%% file: colt2025-sample.tex
\documentclass[final,12pt]{arxiv} 

\title[Decision Making in Hybrid Environments]{Decision Making in Hybrid Environments: \\A Model Aggregation Approach}
\usepackage{times}
\input{commands_colt}
\newcommand{\HL}[1]{{\color{orange}[\text{HL:} #1]}}

\coltauthor{%
 \Name{Haolin Liu} \Email{srs8rh@virginia.edu}\\
 \addr University of Virginia
 \AND
 \Name{Chen-Yu Wei} \Email{chenyu.wei@virginia.edu}\\
 \addr University of Virginia
 \AND
 \Name{Julian Zimmert} \Email{zimmert@google.com}\\
 \addr Google Research
}

\begin{document}

\maketitle

\begin{abstract}%
Recent work by \cite{foster2021statistical, foster2022complexity, foster2023tight} and \cite{xu2023bayesian} developed the framework of \emph{decision estimation coefficient} (DEC) that characterizes the complexity of general online decision making problems and provides a general algorithm design principle. These works, however, either focus on the pure stochastic regime where the world remains fixed over time, or the pure adversarial regime where the world arbitrarily changes over time. For the hybrid regime where the dynamics of the world is fixed while the reward arbitrarily changes, they only give pessimistic bounds on the decision complexity. In this work, we propose a general extension of DEC that more precisely characterizes this case. Besides applications in special cases, our framework leads to a flexible algorithm design where the learner learns over subsets of the hypothesis set, trading estimation complexity with decision complexity, which could be of independent interest. Our work covers model-based learning and model-free learning in the hybrid regime, with a newly proposed extension of the bilinear classes \citep{du2021bilinear} to the adversarial-reward case. In addition, our method improves the best-known regret bounds for linear $Q^\star/V^\star$ MDPs in the pure stochastic regime.  

%
\end{abstract}


\section{Introduction}
The study of decision making problems based on the decision-estimation coefficient (DEC) complexity measure \citep{foster2021statistical} has been shown to characterize the sample complexity of decision making problems. Unlike previous studies on general function approximation that focus more on upper bounds \citep{jiang2017contextual, sun2019model, jin2021bellman, du2021bilinear, zhong2022gec, xie2022role}, DEC not only provides a general algorithm design principle that interleaves model estimation and decision making, but also characterizes the regret lower bound for decision making problems.

Since being proposed by \cite{foster2021statistical}, the DEC framework has been refined or extended in different directions: \cite{chen2022unified} extended the framework to various learning targets beyond no-regret learning, \cite{foster2024model} proposed a related complexity measure for model-free learning, \cite{foster2023tight} refined the gap between the regret upper and lower bounds, \cite{foster2023complexity} applied the approach to multi-agent learning, and \cite{wagenmaker2023instance} established instance-optimal guarantees. Another highly related line of research by \cite{foster2022complexity} and \cite{xu2023bayesian} focus on characterizing the complexity of \emph{adversarial} decision making, where the environment may change its underlying model arbitrarily across time. 

Yet, there are still important open questions remaining in this area. One of them is how to close the gap between the upper and lower bounds in \cite{foster2023tight}. The lower bound only captures the complexity of decision making (i.e., the cost of exploration), while the upper bound includes both the complexity of decision making and model estimation. The nature of the algorithm in \cite{foster2023tight} is a model-based algorithm, so the model estimation error naturally appears in the regret bound. To mitigate such gap, \cite{chen2024assouad} introduced a novel model-estimation complexity measure and provide a complete characterization for bandit learnability. However, their approach still exhibits a significant gap in reinforcement learning.  For general decision making, a common route to reduce estimation errors is to avoid fine-grained model estimation, but only perform \emph{value} estimation, as is the case of model-free value-based learning. However, such a coarser estimation may lead to loss of information that affect the complexity of decision making. It remains open how to characterize such trade-off. 

Another open direction is the study of intermediate environments that sit between pure stochastic environments \citep{foster2021statistical, chen2022unified, foster2023tight, foster2024model} and pure adversarial environments \citep{foster2022complexity, xu2023bayesian}. For learning in Markov decision processes (MDPs), it is well-known that when the transition changes arbitrarily, the worst-case sample complexity could be exponential in the horizon length. Such bound gives a very pessimistic indication for learning in adversarial MDPs. However, it has been shown in some special cases, when only the reward function is adversarial and transition remains fixed, polynomial regret or sample complexity is possible \citep{neu2013online, rosenberg2019online, jin2020learning}. This setting has practical significance, as it models the scenario where the the learner continues to learn new skills while the underlying world remains static \citep{abel2024definition, kumar2023continual}. %
Can we characterize the complexity of such hybrid setting  general cases? Being raised since \cite{foster2022complexity}, this question still remains open.

In this paper, we propose a general framework that contributes to the above two directions. Compared to pure model-based learning approach in \cite{foster2021statistical, foster2023tight} that leads to large estimation complexity, or pure policy-based learning approach in \cite{foster2022complexity} and \cite{xu2023bayesian} that leads to large decision complexity, our framework allows learning over flexible partitions in the joint model and policy space, and trade-off estimation complexity with decision complexity. Under this framework, each partition may aggregate a subset of models or policies within which the learner does not perform finer-grained estimation.  It naturally captures settings such as model-free value learning (aggregating models that correspond to the same value functions), or fixed-transition adversarial-reward settings (aggregating models that have the same transition function), or a combination of them.  We expect that such a general framework may find its more uses beyond the applications in this paper. 

We apply this general framework to the following concrete settings, obtaining new results or recovering existing results from a different viewpoint. Below, hybrid setting refers to MDPs with fixed transition and adversarial reward. 

\begin{enumerate}
    
    \item \textbf{Statistical complexity for the hybrid setting. } 
    We show that the hybrid setting has essentially the same complexity as the fully stochastic setting, as long as the set of reward functions is convex (which is usually the case).  The only overhead for adversarial reward is a $\log(|\Pi||\calP|)$ estimation error as opposed to the $\log(|\calM|) = \log(|\calR||\calP|)$ in the stochastic setting, where $\Pi, \calM, \calP, \calR$ are the policy space, model space, transition space, and reward space, respectively. See \pref{sec: model-based}. 
   
    In scenarios where the adversarial reward function is revealed at the end of each episode (full-information), the same complexity measure applies. However, $\log(|\Pi||\calP|)$ can be improved in MDPs to $\log(|\calA||\calP|)$, where $\calA$ is the action space. See \pref{sec: model-based}. 
    
    \item \textbf{Model-free guarantee for the hybrid setting.} We propose a natural extension of the bilinear class \citep{du2021bilinear} to the hybrid setting. We derive regret bounds that only scale with $\log(|\calF|)$ for when there is full-information loss feedback,  where $\calF$ is the set of value functions, and $|\calF|$ could be much smaller than $|\calP|$. See \pref{sec: model-free adversarial}. 
    
    \item \textbf{New results for the stochastic setting.} For model-free learning in stochastic environments, our framework achieves $\sqrt{T}$ regret for linear $Q^\star/V^\star$ MDPs \citep{du2021bilinear}. To the best of our knowledge, this is the first algorithm that achieves $\sqrt{T}$ regret in this setting without additional assumptions, improving over the previously known $T^{\frac{2}{3}}$ regret attained by the algorithm in \citet{du2021bilinear}.\footnote{Although \citet{du2021bilinear} established an \(O(1/\epsilon^{2})\) sample complexity bound for linear \(Q^{\star}/V^{\star}\), their algorithm does not maintain a confidence set that shrinks with \(\epsilon\). Therefore, the standard policy-elimination reduction from \(O(1/\epsilon^{2})\) sample complexity to \(\sqrt{T}\)-regret cannot be applied. The strongest regret bound we are able to obtain from their method is \(O(T^{2/3})\) by using an explore-then-commit approach.}
    This advances the broader goal of adapting the DEC framework for model-free learning, an objective shared by \cite{foster2024model}. See \pref{sec: stochastic model-free}.   
    
\end{enumerate}

\paragraph{Related work in the hybrid setting}
There is a rich literature in learning fixed-transition adversarial-reward MDPs. Most results are for the tabular case \citep{neu2013online, rosenberg2019online, jin2020learning, shani2020optimistic, chen2021finding}, or the case of linear function approximation \citep{luo2021policy, dai2023refined, sherman2023improved, liu2023towards, kong2023improved, zhao2023learning, li2024improved}. Extensions beyond linear function approximation are scarce, as only seen in \cite{zhao2023learning} and \cite{liu2024beating} for low-rank MDPs. We note that the work by \cite{foster2022complexity} and \cite{xu2023bayesian} for general adversarial decision making can also be applied here, but only achieves tight guarantees when the transition function is known. 

\paragraph{Concurrent work} Independently and concurrently, \cite{chen2025decision} proposed a closely related extension to the DEC framework. Their work and ours, however, focus on rather different applications.


~\\







\section{Preliminaries}
We consider the general decision making problem.  Let $\calM$ be a model class, $\Pi$ be a policy class, and $\calO$ be an observation space. Every model $M: \Pi\to \Delta(\calO)$ specifies a distribution over observations when executing policy $\pi$ in model $M$.  At each round $t=1,2,\ldots, T$, the adversary first chooses a model $M_t\in\calM$ without revealing it to the learner. Then the learner chooses a policy $\pi_t\in\Pi$, and observes an observation $o_t\sim M_t(\cdot |  \pi_t)$. Every model $M$ is associated with a value function $V_M: \Pi\to [0,1]$ that specifies the expected reward of policy $\pi$ in model $M$. We also define $\pi_M$ as the policy that maximize $V_M$.

The regret with respect to the policy sequence $\pi_{1:T}^\star=(\pi_1^\star, \pi_2^\star, \ldots, \pi_T^\star)$ is defined as 
\begin{align*}
    \Reg(\pi_{1:T}^\star) = \sum_{t=1}^T \left(V_{M_t}(\pi_t^\star) - V_{M_t}(\pi_t)\right). 
\end{align*}
For simplicity, in this work, we assume $|\calM|$ and $|\Pi|$ are finite. 


\subsection{Complexity Measures in Prior Work}
\paragraph{DEC} Let $\textsc{co}(\calM)$ be the convex hull of $\calM$. The offset DEC with KL-divergence \citep{foster2021statistical,xu2023bayesian} is defined as 
\begin{align}
\dec_{\eta}^{\textsc{KL}}\left(\calM\right) = \max_{\Bar{M} \in \textsc{co}(\calM)} \min_{p \in \Delta(\Pi)}\max_{M \in \calM} \E_{\pi \sim p}\left[V_M(\pi_M) - V_M(\pi) - \frac{1}{\eta}D_{\textsc{KL}}\left(M(\cdot|\pi), \Bar{M}(\cdot|\pi)\right)\right]. 
\label{eq:DEC}
\end{align}
The DEC aims to achieve the optimal balance between exploitation (minimizing suboptimality $V_M(\pi_M) - V_M(\pi)$) and exploration (maximizing information gain $D_{\textsc{KL}}\left(M(\cdot|\pi), \Bar{M}(\cdot|\pi)\right)$) under the worst possible environment.

\paragraph{AIR} \citet{xu2023bayesian} introduced AIR as a framework to analyse the regret of arbitrary algorithms and derive upper bounds on the complexity of online learning problems. Given $\rho\in\Delta(\Pi)$ and $\eta>0$, $\air:\Delta(\Pi)\times\Delta(\Pi \times \calM)\rightarrow \bbR$ is defined as
\begin{align*}
    \air_{\rho,\eta}(p,\nu) = \E_{\pi\sim p}  \E_{(M, \pi^\star) \sim \nu} \E_{o\sim M(\cdot|\pi)}\left[ 
 V_M(\pi^\star) - V_M(\pi) - \frac{1}{\eta}\KL\left(  \nu_{\bo{\pi^\star}}(\cdot | \pi, o), \rho \right) \right].
\end{align*}
Here, $\nu$ is a distribution over $\calM\times \Pi$, and $\nu_{\bo{\pi^\star}}(\cdot|\pi,o)$ is a posterior distribution over $\pi^\star$ given $(\pi,o)$, which can be calculated as
\begin{align}
    \nu_{\bo{\pi^\star}}(\pi^\star|\pi,o) = \frac{\sum_{M}\nu(M,\pi^\star)M(o|\pi)}{\sum_{M,\pi'}\nu(M,\pi')M(o|\pi)}.  \label{eq: for example}
\end{align}
The subscript $\bo{\pi^\star}$ in the notation $\nu_{\bo{\pi^\star}}(\cdot|\pi,o)$ specifies what variable the `~$\cdot$~' represents. We place it in the subscript and make it bold to distinguish it from a realized value of that random variable. When the variable is clear, we omit this subscript. For example, the left-hand side of \pref{eq: for example} can be simply written as $\nu(\pi^\star|\pi, o)$. 
Similar to DEC in \pref{eq:DEC}, AIR also seeks the optimal balance between minimizing sub-optimality gap and maximizing information gain. The key difference is that AIR considers an arbitrary comparator policy $\pi^\star$ instead of $\pi_M$ as in DEC, with exploration measured by the information gain associated with $\pi^\star$. The flexibility of the comparator policy enables AIR to handle the adversarial setting.  

\paragraph{MAIR} For model based stochastic environments, given $\rho\in\Delta(\calM)$ and $\eta$, \citet{xu2023bayesian} further define $\mair:\Delta(\Pi)\times\Delta(\calM)\rightarrow \bbR$ as 
\begin{align*} 
    \mair_{\rho,\eta}(p,\nu) = \E_{\pi\sim p}  \E_{M \sim \nu}\E_{o\sim M(\cdot|\pi)}\left[ 
 V_M(\pi_M) - V_M(\pi) - \frac{1}{\eta}\KL\left(  \nu_{\bo{M}}(\cdot | \pi, o), \rho \right) \right].
\end{align*}
Here $\nu$ is only a distribution over $\calM$. This is because in stochastic setting, it suffices to compare against $V_M(\pi_M)$, making it unnecessary to consider a joint distribution over $\calM \times \Pi$. \cite{xu2023bayesian} also shows that MAIR is tightly related to DEC.

\ \\

\noindent\textbf{EXO} by \cite{lattimore2020exploration} and \cite{foster2022complexity} is closely related to AIR, and also serves as a useful complexity measure and algorithm design principle for adversarial decision making, with the additional benefit of getting high-probability bounds. In this work, we base our presentation in AIR due to its conciseness, but everything could be translated to the EXO framework.

\subsection{Markov Decision Processes} 
A Markov decision process (MDP) is defined by a tuple $(\calS, \calA, P, R, H, s_1)$, where $\calS$ is the state space, $\calA$ is the action space, $P:\calS\times\calA\rightarrow\Delta(\calS)$ the transition kernel, $R:\calS\times\calA\rightarrow[0,1]$ the reward function,  $H$ the horizon and $s_1$ the starting state.
A learner interacts with an MDP for $H$ rounds $h=1,\dots, H$.
In every round, the learner observes the current state $s_h$ and selects an action $a_h\in\calA$. The state is transitioning into next state via the transition kernel $s_{h+1}\sim P(\cdot|s_h,a_h)$ and receives the reward $R(a_h,s_h)$.
We assume that the reward functions are constraint such that $\sum_{h=1}^HR(s_h,a_h)\leq 1$ for any policy almost surely. Given a policy $\pi:\calS\to\calA$, the $Q$-function and value function are defined by $Q^\pi_h(s, a)=\E^\pi[\sum_{h'=h}^{H-1}R(s_h,a_h)\,|\,s_h=s,a_h=a]$, $V^\pi_h(s)=Q^\pi_h(s,\pi(s))$ respectively. The value and $Q$-function of an optimal policy $\pi^\star$ are abbreviated with $V^\star/Q^\star$. We use $Q_h^\pi(s,a;M)$ and $Q_h^\star(s,a;M)$ where $M=(P,R)$ to denote the value functions under model $M$.

\section{A General Framework}\label{sec: general frame}
In order to formalize mixed-adversarial and stochastic regimes, we propose the following framework.
Let $\Phi$ be a collection of disjoint subsets of $\calM\times\Pi$. That is, 1) each element $\phi\in\Phi$ is a subset of $\calM\times \Pi$, and 2) for any $\phi_1, \phi_2\in \Phi$, if $\phi_1\neq \phi_2$, then $\phi_1\cap \phi_2=\emptyset$.\footnote{The requirement that subsets in $\Phi$ be disjoint is not really necessary, as models can be duplicated if needed. However, we adopt this assumption as it holds in all our applications and simplifies the notation.}
\begin{definition}\label{def: restricted env} A $\Phi$-restricted environment is an (adversarial) decision making problem in which the environment commits to $\phi^\star\in\Phi$ at the beginning of the game and henceforth selects $(M_t,\pi^\star_t)\in \phi^\star$ in every round $t$ arbitrarily based on the history.
\end{definition}

\paragraph{Example}  
Let $\calM=\calP\times\calR$ decomposes into the set of transitions and rewards of a class of MDPs respectivly. Set $\Phi=\{\phi_{P,\pi^\star}|(P,\pi^\star)\in\calP\times\Pi\}$, where $\phi_{P,\pi}=\{((P,R),\pi )|R\in\calR\}$. The $\Phi$-restricted environment is exactly the setting of an MDP with fixed transition, fixed comparator and adversarial rewards.

To characterize the complexity of $\Phi$-restricted environments, we define for $\rho\in\Delta(\Phi)$, $\eta>0$, 
\begin{align*}
    \air^{\Phi}_{\rho,\eta}(p,\nu) = \E_{\pi\sim p} \E_{(M,\pi^\star)\sim \nu} \E_{o\sim M(\cdot|\pi)}\left[V_M(\pi^\star) - V_M(\pi) - \frac{1}{\eta} \KL(\nu_{\bo{\phi}}(\cdot|\pi, o), \rho)\right],  
\end{align*}
where $\nu_{\bo{\phi}}(\cdot|\pi,o)$ is the posterior over $\Phi$ given $(\pi, o)$, which can be calculated as 
\begin{align*}
    \nu(\phi|\pi,o) = \frac{\sum_{(M,\pi^\star)\in \phi} \nu(M,\pi^\star) M(o|\pi) }{\sum_{\phi'\in\Phi}\sum_{(M,\pi^\star)\in \phi'} \nu(M,\pi^\star) M(o|\pi)}. 
\end{align*}
Again, this follows our convention that the $\bo{\phi}$ in the subscript of $\nu_{\bo{\phi}}(\cdot|\pi,o)$ is only an indicator of the variable represented by `~$\cdot$~', rather than a realized value, and it will be omitted when it is clear.  

$\air^\Phi$ recovers $\air$ and $\mair$. Specifically, $\air^\Phi$ recovers $\air$ when every $\phi$ corresponds to a single $\pi^\star\in\Pi$ and all $M\in\calM$, i.e., $\Phi=\{\phi_{\pi^\star}: \pi^\star\in\Pi\}$ where $\phi_{\pi^\star}=\{(M,\pi^\star): M\in \calM\}$. $\air^\Phi$ recovers $\mair$ when every $\phi$ corresponds to a single model $M$ and the associated best policy $\pi_M$, i.e., $\Phi=\{\phi_M:~ M\in\calM\}$ where $\phi_M = \{(M,\pi_M)\} = \{(M,\argmax_{\pi\in\Pi} V_M(\pi) )\}$. We define $\joint = \left(\bigcup_{\phi\in\Phi} \phi\right) \subset \calM\times \Pi$.

With this definition, we propose the general algorithm in \pref{alg:general AIR}. Similar to \cite{xu2023bayesian}, we can decompose the objective in \pref{alg:general AIR} as the following terms: 
\begin{align*}
   \underbrace{\E_{\pi\sim p} \E_{(M,\pi^\star)\sim \nu} \left[V_M(\pi^\star) - V_M(\pi)\right]}_{\text{expected regret}} - \frac{1}{\eta} \underbrace{\E_{\pi\sim p} \E_{(M,\pi^\star)\sim \nu} \E_{o\sim M(\cdot|\pi)} \left[ \KL(\nu_{\bo{\phi}}(\cdot|\pi, o), \nu_{\bo{\phi}})\right]}_{\text{information gain}} - \frac{1}{\eta}\underbrace{\KL(\nu_{\bo{\phi}}, \rho_t)}_{\text{regularization}}
\end{align*}
where $\nu_{\bo{\phi}}(\phi) = \sum_{(M,\pi^\star)\in\phi} \nu(M,\pi^\star)$. The algorithm proceeds as follows. 
Upon receiving a reference distribution $\rho_t\in\Delta(\Phi)$ (an estimation from the previous round about~$\phi^\star$), the algorithm finds a policy distribution $p_t$ that simultaneously tries to minimize the expected regret and maximize the information gain about $\phi^\star$ against the worst case choice of the world distribution~$\nu_t\in\Delta(\calM\times \Pi)$ that is not too far from the previous estimation $\rho_t$. At the end of round $t$, the algorithm calculates $\rho_{t+1}$ as a posterior of $\nu_t$ with the new observation $(\pi_t, o_t)$, only stores this information, and forgets everything else. 

The algorithm $\air$ by \cite{xu2023bayesian} only maintains an estimation of the optimal policy $\pi^\star$, while their $\mair$ maintains an estimation of the underlying model $M^\star$. Our framework provides a natural generalization of them by allowing the target of estimation be subsets of $\calM\times \Pi$. For example, in MDPs with fixed transition and adversarial reward, we jointly estimate $(P^\star,\pi^\star)$, while for model-free value learning in stochastic MDPs, we estimate $(f^\star, \pi_{f^\star})$ where $f^\star$ is the true optimal value function and $\pi_{f^\star}$ is its corresponding optimal policy. These are achievable by properly choosing $\Phi$, as will be elaborated in the following sections. 

\begin{algorithm}[t]
\caption{General Algorithm} \label{alg:general AIR}
    Let $\rho_1(\phi) = \frac{1}{|\Phi|}$ for all $\phi \in \Phi$. \\
    \For{$t=1,2,\ldots, T$}{
       Find distributions $p_t\in\Delta(\Pi)$ and $\nu_t\in\Delta(\Psi)$ that solve the saddle point of 
       \begin{align*}
           \min_{p\in\Delta(\Pi)} \max_{\nu\in \Delta(\Psi)}  \air^{\Phi}_{\rho_t,\eta}(p,\nu). 
       \end{align*}
       Sample decision $\pi_t\sim p_t$ and observe $o_t\sim M_t(\cdot|\pi_t)$. \\
       Update $\rho_{t+1}(\phi)=\nu_{t}(\phi|\pi_t,o_t)$ for all $\phi\in\Phi$.  
       
    }
\end{algorithm}

The regret bound achieved by \pref{alg:general AIR} is stated in the following theorem. 

\allowdisplaybreaks
 \begin{theorem}\label{thm: general}
 For a $\Phi$-restricted environment, there is an algorithm such that 
\begin{align*}
    \E\left[\Reg(\pi^\star_{1:T})\right] \leq \frac{\log|\Phi|}{\eta} + T \max_{\rho \in\Delta(\Phi)}\max_{\nu\in\Delta(\joint)}\min_{p\in\Delta(\Pi)} \air^\Phi_{\rho,\eta}(p,\nu). 
\end{align*}
\end{theorem}
The bound in \pref{thm: general} consists of two parts. The first part $\frac{\log|\Phi|}{\eta}$ captures the \emph{estimation complexity} over $\Phi$, while the other part $T\cdot \air^\Phi$ captures the \emph{decision complexity} that quantifies the cost of exploitation-exploration tradeoff. A finer partition $\Phi$ leads to larger estimation complexity but smaller decision complexity, and a coarser partition leads to the opposite. Since a coarser partition makes the environment strictly more powerful, the learner has the freedom to find the best trade-off over coarser partitions, while still ensuring a bound on the environment at hand.

\begin{remark}
All concrete applications studied in this paper actually consider the standard regret $\Reg(\pi^\star)$ where the regret comparator $\pi^\star$ is fixed over time. That requires each $\phi\in\Phi$ to have a unique $\pi^\star$.  The generality of our framework in allowing changing comparators $\pi^\star_{1:T}$ may be of independent interest and is left for future investigation. We note that this theorem does not violate the impossibility result for general time-varying comparator $\pi^\star_{1:T}$, as in that case the $\air^\Phi$ term in \pref{thm: general} could be of order $\Omega(1)$. This is discussed more deeply in \pref{lem: air by dec}.   
\end{remark}


\subsection{Relation with DEC}
As noted in \cite{foster2022complexity} and \cite{xu2023bayesian}, for pure adversarial decision making, the decision complexity is governed by the DEC of the convexified model class $\co(\calM)$. In the $\Phi$-restricted environment (\pref{def: restricted env}) that we consider, we show that the complexity is upper bounded by the DEC of a refined convexified model class taking the partition into account. 
\begin{definition}
The $\Phi$-aware convexification of the model class $\calM$ is defined as 
$$\bar\calM(\Phi):=\bigcup_{\phi\in\Phi}\co(\{M\in\calM~|~\exists \pi \text{\ such that\ }(M,\pi)\in\phi\}).$$
\end{definition}
In other words, models are convexified within each subset $\phi$. This describes the complexity of the $\Phi$-restricted environment as follows.
\begin{lemma}\label{lem: air by dec}
    It holds that
    $\max_{\rho \in \Delta(\Phi)}\max_{\nu\in\Delta(\joint)}\min_{p\in\Delta(\Pi)}  
\air^\Phi_{\rho,\eta}(p,\nu)\leq\dec^{\textsc{KL}}_\eta(\bar\calM(\Phi))$, \textbf{unless} $\Phi$ is an environment with adaptive comparator where sublinear regret is impossible. In this case
    $\max_{\rho \in \Delta(\Phi)}\max_{\nu\in\Delta(\joint)}\min_{p\in\Delta(\Pi)}  
\air^\Phi_{\rho,\eta}(p,\nu)=\Omega(1)$ independently of $\eta$ 
.
\end{lemma}
While this is only providing an upper bound, it is actually tight in many environments of interest. Let a \emph{fixed comparator game} be such that all $\phi\in\Phi$ contain a single policy $\pi^\phi$. 
\begin{lemma}
\label{lem: dec equiv}
In all fixed comparator games if the choice of model is independent of the choice of comparator then $\max_{\rho \in \Delta(\Phi)}\max_{\nu\in\Delta(\joint)}\min_{p\in\Delta(\Pi)}  
\air^\Phi_{\rho,\eta}(p,\nu)=\dec^{\textsc{KL}}_\eta(\bar\calM(\Phi))$ .
\end{lemma}
The choice of model being independent of the choice of comparator means that $\Phi=\Theta\times\Pi$ for an arbitrary disjoint partition $\calM=\bigcup_{\theta\in\Theta}\theta$ of models. This is in fact the case for most of our applications except for the application in \pref{sec: model-free adversarial}. 

Note that in stochastic regimes, we often trim the set to pairs of model and its optimal policy (see \pref{sec: stochastic model-free}). This trimming neither impacts the maxmin value of $\air^\Phi$ nor the convexification (see \pref{lem: trimming} in the appendix), so equality still holds.

\subsection{Informed Comparator Case}\label{sec: full-info}
In this section, we consider the fixed-comparator regime where the choice of model is independent of the comparator, that means there is a partition $\Theta$ of $\calM$, such that $\Phi=\Theta\times\Pi$. 

The complexity $|\Phi|$ is $|\Theta|\cdot|\Pi|$, which can be highly sub-optimal when $|\Theta|\ll|\Pi|$.
To overcome this limitation, we propose to split the online learning problem into two modified subgames defined as follows. 

\emph{Simultaneous learning game:} The environment mainly follows the same protocol as before, but the learner selects a meta-policy $(\pi_t^\theta)_{\theta\in\Theta}\in\Pi^{|\Theta|}$ in every round instead. Essentially, the learner is allowed to choose a dedicated policy for \emph{each} model subset $\theta\in\Theta$. However, the learner loses control over the observation policy $\pi_t$. The regret is only evaluated on the choice of $\pi_t^{\theta^\star}$, where $\theta^\star$ is the unknown ground truth decided by the environment at the beginning (see \pref{eq: decomposition}). 

\emph{Informed comparator game:} At the beginning of the game, the environment still secretly decides $\theta^\star\in\Theta$. In each round, additionally to deciding $M_t\in \theta^\star$, the environment also chooses $(\pi_t^\theta)_{\theta\in\Theta}\in\Pi^{|\Theta|}$, which is revealed to the learner before they make their decision $\pi_t$. The comparator in round $t$ is $\pi_t^{\theta^\star}$. 

These games are motivated by the following regret decomposition.
\begin{align}
    \E\left[\Reg(\pi^\star)\right] = \underbrace{\E\left[\sum_{t=1}^TV_{M_t}(\pi^\star)-V_{M_t}(\pi_t^{\theta^\star})\right]}_{ \text{Regret in \emph{Simultaneous learning game}}}+\underbrace{\E\left[\Reg(\pi^{\theta^\star}_{1:T})\right]}_{ \text{Regret in \emph{Informed comparator game}}}\label{eq: decomposition}\,.
\end{align}
In general, the loss of control over the observation policy makes the \emph{Simultaneous learning game} hard to solve. However, when the observation is independent of the observation policy, this game is strictly easier than the vanilla setting. We present a specific example at the end of the section.

Let us first show that the \emph{Informed comparator game} can be solved independently of the size of the policy space.
Denote $\pi^\Theta=(\pi^\theta)_{\theta\in\Theta}$ and define for a distribution $\nu\in\Delta(\calM)$
\begin{align*}
    \infoair^\Theta_{\rho,\eta}(p,\nu,\pi^\Theta) = \E_{\pi\sim p}\E_{\theta^\star \sim \nu} \E_{M \sim \nu(\cdot|\theta^\star)} \E_{o\sim M(\cdot|\pi)}\left[V_M(\pi^{\theta^\star}) - V_M(\pi) - \frac{1}{\eta} \KL(\nu_{\bo{\theta}}(\cdot|\pi, o), \rho)\right]. 
\end{align*}
The following theorem refers to the use of \pref{alg:general AIR full info}.
\begin{theorem}\label{thm: informed}
   In the informed comparator setting, there exists an algorithm with regret
   \begin{align*}
    \E\left[\Reg(\pi^{\theta^\star}_{1:T})\right] &\leq \frac{\log|\Theta|}{\eta} +  T\max_{\pi^\Theta\in\Pi^\Theta}\max_{\rho \in\Delta(\Theta)}\max_{\nu\in\Delta(\calM)}\min_{p\in\Delta(\Pi)} \infoair^\Phi_{\rho,\eta}(p,\nu, \pi^\Theta) \\
    &= \frac{\log|\Theta|}{\eta} + T\cdot \dec^{\textsc{KL}}_\eta(\bar\calM(\Theta)). 
\end{align*}
\label{thm:general-full}
\end{theorem}
Since $\bar{\calM}(\Theta)=\bar{\calM}(\Phi)$, this is the same bound as \pref{thm: general}, except that we only suffer the penalty $\log(|\Theta|)$ instead of $\log(|\Theta||\Pi|)$.

Finally, let us present a case in which the \emph{Simultaneous learning game} is efficiently solvable. We define full-information feedback such that the learner receives a candidate model $M_t^\theta$ for every $\theta\in\Theta$ after round $t$, with the guarantee that $M_t=M_t^{\theta^\star}$.
\begin{theorem}\label{thm: simultaneous full-info}
    Let $\Theta$ describe an environment over $H$-staged MDPs with action set $\calA$ and full-information reward feedback. For any $\gamma>0$ there exists an algorithm in the \emph{Simultaneous learning game} with regret bounded by $\left(\frac{\log|\calA|}{\gamma} + \gamma T\right) H$. 
\end{theorem}
In $H$-staged MDPs with state set $\calS$, the policy space we consider can be as large as $|\Pi|=|\calA|^{|\calS|}$
In cases where $\calS$ is very large, i.e. $|\calS|\gg\max\{|\calA|,|\Theta|\}$, combining \pref{thm: informed} and \pref{thm: simultaneous full-info} yield significantly better bounds than applying \pref{thm: general} directly.

\begin{algorithm}[t]
\caption{General Algorithm with Informed Comparator} \label{alg:general AIR full info}
    Let $\rho_1(\theta) = \frac{1}{|\Theta|}$ for all $\theta \in \Theta$. \\
    \For{$t=1,2,\ldots, T$}{
       Receive $\pi_t^\Theta=(\pi_t^\theta)_{\theta\in\Theta}$ from the environment. \\
       Find a distribution $p_t$ of $\pi$ and a distribution $\nu_t$ of $\calM$ that solve the saddle-point of 
       \begin{align*}
           \min_{p\in\Delta(\Pi)} \max_{\nu\in \Delta(\calM)} \infoair^\Theta_{\rho_t,\eta}(p,\nu, \pi_t^\Theta)  
       \end{align*}
       Sample decision $\pi_t\sim p_t$ and observe $o_t\sim M_t(\cdot|\pi_t)$. \\
       Update $\rho_{t+1}(\theta)=\nu_{t}(\theta|\pi_t,o_t)$ for all $\theta\in\Theta$.  
       
    }
\end{algorithm}
\section{Model-Based RL in Adversarial MDPs with Fixed Transition}\label{sec: model-based}
In this section, we consider the setting where the reward is adversarially chosen but the transition is fixed.
We decompose the model space $\calM = \calR \times \calP$ where $\calR$ is the reward space and $\calP$ is the transition space. 

This corresponds to the partition  $\Phi_{\textsc{A}} = \left\{\phi_{{\pi^\star, P^\star}}: \pi^\star \in \Pi, P^\star \in \calP\right\}$ where \sloppy$\phi_{{\pi^\star, P^\star}} = \left\{( (P^\star, R), \pi^\star): R \in \calR\right\}$. Moreover, now $\joint_{\sA} = \left(\bigcup_{\phi\in\Phi_\textsc{A}} \phi\right) = \calM\times \Pi$. 
By definition, we have $|\Phi_A|=|\Pi|\cdot|\calP|$ and $\bar\calM(\Phi_A)=\calP\times\co(\calR)$. Note that for typical applications, $\calR$ is convex and hence $\bar\calM(\Phi_A)=\calM$.

\begin{corollary}
If the reward space $\calR$ is convex, there exists an algorithm for the hybrid problem with regret
\begin{align*}
    \E[\Reg(\pi^\star)] \leq \frac{\log\left(|\Pi||\calP|\right)}{\eta} +T\cdot\dec^{\textsc{KL}}_\eta(\calM)\,.
\end{align*} 
\end{corollary}
The proof follows directly from combining \pref{thm: general} and \pref{lem: air by dec}. Furthermore, when there is full-information reward feedback, 
\begin{corollary}
If the reward space $\calR$ is convex and the learner observe $R_t$ after round $t$, then there exists an algorithm with regret
$$
    \E[\Reg(\pi^\star)]\leq \left(\frac{\log(|\calA|)}{\gamma}+\gamma T\right)H+\frac{\log|\calP|}{\eta} + T\cdot \dec^{\textsc{KL}}_\eta(\calM)\,.$$
\end{corollary}
This is a direct corollary of \pref{thm: informed}, \pref{thm: simultaneous full-info} and \pref{lem: air by dec}.

\noindent \textbf{Example\ \  } For the model-based low-rank MDP setting with fixed transition and adversarial linear rewards studied in \cite{liu2024beating}, these corollaries directly improve the best-known regret guarantees from $T^\frac{2}{3}$ to $\sqrt{T}$ for both full-information and bandit feedback.

\section{Model-Free RL in Stochastic MDPs}\label{sec: stochastic model-free}
In this section, we apply the general framework to model-free RL in stochastic MDPs. In model-free RL with value function approximation, the learner is provided with a function set $\calF$ that contains possible value functions of the world. We define $\Phi$ as $\Phi=\{\phi_f: f\in\calF\}$ where $\phi_f=\{(M, \pi_M):~ M \text{\ induces\ } f\}$.  With \pref{alg:general AIR}, we are able to obtain the regret bound 
\begin{align*}
    \E\left[\Reg(\pi^\star)\right] \leq \frac{\log|\calF|}{\eta} + T\cdot \dec^{\textsc{KL}}_\eta (\bar\calM(\Phi)). 
\end{align*}
In fact, model-free learning in stochastic MDPs is easier than the $\Phi$-restricted environment defined in \pref{def: restricted env}, as the adversary cannot modify the underlying model in every round. However, the general framework offers a means to reduce estimation complexity from $\log|\calM|$ to $\log|\calF|$. If we can show that aggregation does not significantly increase decision complexity, this approach could provide a viable path to improving the regret bound.

Existing literature has identified general classes where model-free RL allows for $\log|\calF|$ estimation complexity and bounded decision complexity. One of the most general classes is the \emph{bilinear class} \citep{du2021bilinear}, where the \emph{bilinear rank} serves as the decision complexity. 
While we are currently unable to relate  $\dec^{\textsc{KL}}_\eta (\bar\calM(\Phi))$ to bilinear rank in full generality, we identify special cases where such a connection holds.  Specifically, in the linear $Q^\star/V^\star$ setting---a subclass of bilinear class---we show that $\dec^{\textsc{KL}}_\eta (\bar\calM(\Phi))$ can be upper bounded by the bilinear rank. We formally define the linear $Q^\star/V^\star$ setting as the following: 
\begin{definition}
    A class of MDPs satisfies linear $Q^\star/V^\star$ if there are known feature vectors $\varphi: \calS\times \calA\to \mathbb{R}^d$ and $\psi: \calS\to \mathbb{R}^d$ and known function class $\calF=\{(\theta,w)\}\subset \mathbb{R}^{d+d}$ such that $Q^\star(s,a) = \varphi(s,a)^\top \theta^\star$ and $V^\star(s) = \psi(s)^\top w^\star$ for some $(\theta^\star, w^\star)\in\calF$. 
\end{definition}
We assume $|\calF|$ is finite for simplicity. It is straightforward to extend the result to the case where $\calF$ lies in a bounded region in $\mathbb{R}^{d+d}$. With this, we let $\Phi = \{\phi_{\theta,w}~:~(\theta,w) \in \calF\}$, where 
\begin{align*}
     \phi_{\theta,w} = \{(M, \pi_M)~:~ Q^\star(s,a; M) = \varphi(s,a)^\top \theta, \quad V^\star(s; M) = \psi(s)^\top w\}. 
\end{align*}
Then we have the following bound for $\dec^{\textsc{KL}}_\eta(\bar\calM(\Phi))$. 
\begin{theorem}\label{thm: linear QV}
    In the linear $Q^\star/V^\star$ setting, $\dec^{\textsc{KL}}_\eta(\bar\calM(\Phi))\leq 4\eta dH^2$. 
\end{theorem}
From \pref{thm: linear QV}, we can get $\sqrt{T}$ regret for linear $Q^\star/V^\star$ MDPs for the first time. Previous works on this setting either only get sample complexity bound \citep{du2021bilinear} or require additional Bellman completeness assumption \citep{jin2021bellman}. Establishing similar results for general bilinear classes or other existing complexity measures for model-free RL is left as future work. We provide more discussions in \pref{sec: discussion}.

\section{Model-Free RL in Adversarial MDPs with Fixed Transition and Full Information}\label{sec: model-free adversarial} 

In this section, we study model-free RL in the hybrid setting with full-information reward feedback. The algorithm used in this section deviates from the general algorithm in \pref{alg:general AIR}, but it is still based on the idea of learning over partitions of $\calM\times \Pi$.  To the best of our knowledge, there is no general value function approximation scheme defined for the hybrid setting in the literature. Therefore, we propose the following framework. 

\begin{definition}[Value function approximation for the hybrid setting]\label{def: func approx}
    Assume for each policy $\pi\in\Pi$, there is a function class $\calF^\pi$. Each function $f\in\calF^\pi$ is a ternary function mapping $\calS \times \calA \times \calR\to \mathbb{R}$. We say a transition $P$ induces $f\in\calF^\pi$ if $Q^{\pi}(s,a;(P,R)) = f(s,a, R)$\footnote{Similar to \cite{foster2024model}, for the sake of simplicity, we adopt a less general definition where ``functions'' only represent $Q$-functions. The results can be readily extended to more general settings where functions may additionally represent e.g. $V$-functions. } for any $s, a, R$. 
    With abuse of notation, for any $\pi\in\Pi$ and $f\in\calF^\pi$, we write $f(\pi; R) = f(s_0, \pi(s_0), R)$, indicating that it is the value of policy $\pi$ predicted by function $f$ under reward $R$. Define $\calF =  \cup_{\pi\in\Pi} \calF^{\pi}$. We assume $f(\pi, R) \in [0,1]$ for any $f, \pi, R$, which matches the assumption on value function $V_M$.
\end{definition}

Compared to standard value function approximation for the stochastic setting (e.g., \cite{du2021bilinear}), the framework in \pref{def: func approx} is different in two ways: 
1) In the usual stochastic setting, it is a full model $M=(P,R)$ that underlies (or induces) a function. In other words, given a model $(P,R)$, there is a unique function $f$ corresponding to it. However, in the hybrid setting, it is a transition $P$ that induces a function. The reward $R$, on the other hand, is an input to the function. This is natural because reward functions change over time and no single function can capture it. 
2) In the usual stochastic setting, the functions only represent the optimal $Q$-functions, while in the hybrid setting, the functions represent $Q$-functions of all policies.  This is also natural because even under a fixed transition, every policy has the possibility to become the optimal policy when the reward can vary. 



In the model-free hybrid setting where the learner is equipped with a function class $\calF$ described in \pref{def: func approx}, we make the partition be  
$\Phi=\left\{ \phi_{\pi, f}: \pi\in\Pi, f\in \calF^\pi\right\}$, where $\phi_{\pi,f}=\{(P, R, \pi): P \text{\ induces\ } f, \,  R \in \mathcal{R}\}$ for any $\pi\in\Pi$ and 
$f\in\calF^\pi$. 
Since every partition corresponds to a single $f$ and policy $\pi$, for any partition $\phi\in\Phi$, we denote its corresponding $f$ as $f^\phi$, and $\pi$ as $\pi^\phi$. 

Note that, unlike the applications in previous sections, here we have a \emph{policy-dependent} partition over models. This is why, in \pref{sec: general frame}, we construct $\Phi$ based on the joint model and policy class rather than solely on the model class.

Next, we further make structural assumptions on the model and function classes. Specifically, we provide a natural extension of the bilinear class \citep{du2021bilinear} to the hybrid setting described in \pref{def: func approx}. 

\begin{definition}[Bilinear class for the hybrid setting]
\label{def:adv bilinear}
A hybrid function class $\calF$ (defined in \pref{def: func approx}) induced by a model class $\calM=\calP\times\calR$ has bilinear rank $d$ if there exists functions $X_h: \Pi \times \mathcal{P} \rightarrow \mathbb{R}^d$ and  $W_h: \mathcal{F} \times \mathcal{R} \times \mathcal{P} \rightarrow \mathbb{R}^d$ for all $h\in[H]$ such that for any $\pi\in\Pi$, $f \in \calF^\pi$, $R \in \mathcal{R}$, and $ P \in \mathcal{P}$, we have
    \begin{align*}
        \left|f(\pi; R) - V_{P, R}(\pi)\right| \leq \sum_{h=1}^H \left|\langle X_h(\pi; P), W_h(f, R; P)\rangle\right|.
    \end{align*}
    Moreover, there exists an estimation policy mapping function $\esttt: \Pi \rightarrow \Pi$, and for every $h \in [H]$, there exists a discrepancy function $\ellest_h: \mathcal{F} \times \calO \times \calR  \rightarrow \mathbb{R}$ such that for any $\pi\in\Pi, f \in \calF^\pi$, and $f'\in \calF$, 
    \begin{align*}
        \left|\langle X_h(\pi; P), W_h(f', R; P) \rangle\right| = \left|\E^{\pi \,\circ_h\, \esttt(\pi),\, P} \left[\ellest_h(f', o_h, R)\right]\right|
    \end{align*}
    where $\mathcal{O}$ is the space of per-step observation $(s,a,s')$, and $o_h = (s_h,a_h,s_{h+1})$. $\E^{\pi, P}[\cdot]$ denotes the expectation over the occupancy measure generated by policy $\pi$ and transition $P$. Let $\pi \circ_h \esttt(\pi)$ denote playing $\pi$ for the first $h-1$ steps and play policy $\esttt(\pi)$\, at the $h$-th step.
\end{definition}
\pref{def:adv bilinear} encompasses numerous interesting instances, such as low-rank MDPs and low-occupancy MDPs \citep{du2021bilinear}. Their learnability in the hybrid setting remains largely unexplored. Further discussion on these specific instances can be found in \pref{app:model-free full proof}.

We will show that MDPs that satisfy \pref{def:adv bilinear} is learnable in the hybrid setting with model-free guarantees. To emphasize more on the idea of partition, we use the partition notation to describe our \pref{alg:MAF}. For partition $\phi = \phi_{\pi, f}$ where $\pi \in \Pi$ and $f \in \calF^{\pi}$, we write $X_h(\phi; P) = X_h(\pi; P)$, $W_h(\phi; P) = W_h(f, R; P)$, $\ellest_h(\phi, o_h, R) 
 = \ellest_h(f, o_h, R)$, and $\pi_{est}^{\phi} = \esttt(\pi)$. 
We define the bilinear divergence
\begin{align*}
    D_{bi}^{\pi}(\phi||P, R) = \sum_{h=1}^H \left(\E^{\pi, P}\left[\ellest_h(\phi, o_h, R)\right]\right)^2.
\end{align*}

\begin{algorithm}[t]
    \caption{Model-free learning for the hybrid setting with full information}
    \label{alg:MAF}
    \textbf{Input:} Partition set $\Phi =  \left\{\phi_{\pi, f}: \pi \in \Pi, f \in \calF^\pi\right\}$,  $p_1(\phi) = \frac{1}{|\Phi|},\, \forall \phi \in \Phi$, 
   epoch length $\tau$, and learning rate $\gamma$, $\eta$. \\
    \For{\textup{\textbf{epoch}} $k=1, 2, \ldots, \frac{T}{\tau}$}{
         Sample $\pi_k  \sim p_k$. \\
         \For{$i=1,\ldots, \tau$}{
         Execute $\pi_k$, obtain trajectory $(o_{k,1}^i, \ldots, o_{k,H}^i)$ where $o_{k,h}^i = (s_{k,h}^i, a_{k,h}^i, s_{k,h+1}^i)$, and observe the full reward function $R_k^i$. 
         }
         
         Define $R_k = \frac{1}{\tau}\sum_{i \in [\tau]} R_k^i$ as the average reward in the $k$-th epoch.
        
         Compute $\rho_{k+1} \in \Delta(\Phi)$ as 
         \begin{align}
             \rho_{k+1}(\phi) \propto \rho_{k}(\phi) \exp\left(\gamma \eta f^{\phi}(\pi^{\phi}; R_k) -  \gamma \sum_{h=1}^H \left(\frac{1}{\tau} \sum_{i \in [\tau]}\ellest_h(\phi, o_{k,h}^i, R_k)\right)^2\right).
             \label{eq:opt posterior}
         \end{align}
        
        
        Solve the following minimax optimization for $p_{k+1}$
        \begin{align}
            p_{k+1} = \min_{p \in \Delta(\Pi)}\max_{(P,R) \in \calM} \E_{\pi' \sim p}\E_{\phi \sim \rho_{k+1}}\left[f^{\phi}(\pi^{\phi}; R) -V_{(P,R)}(\pi') - \frac{1}{8\eta } D_{bi}^{\pi'}\left(\phi||P, R\right)\right].
            \label{eq:optimis minmax}
        \end{align}
        }
\end{algorithm}

For this setting, we devise \pref{alg:MAF}, which is inspired by \cite{foster2024model} who studied the problem in the stochastic setting. The episodes are divided into epochs. At the beginning of each epoch, a reference distribution $\rho \in \Delta(\Phi)$ is selected based on optimistic posterior sampling \citep{zhang2022feel} with estimated bilinear divergence (\pref{eq:opt posterior}). Subsequently, the algorithm computes a behavior policy distribution $p_t$ through a minimax optimization \pref{eq:optimis minmax} analogous to the DEC objective defined in \pref{eq:DEC}. However, unlike the standard DEC that uses the optimal value $V_M(\pi_M)$ as the comparator where $M$ is the worst-case model, \pref{eq:optimis minmax} adopts $f^\phi(\pi^\phi, R)$ as the comparator, where $\phi$ is drawn from the optimistic posterior determined in \pref{eq:opt posterior}. As in \cite{foster2024model}, such change of comparator enables the effective use of optimism to derive improved model-free guarantee. 
With the policy distribution obtained from \pref{eq:optimis minmax}, a policy is sampled from it and executed throughout the entire epoch to gather data for the next iteration. The guarantee of \pref{alg:MAF} is presented in \pref{thm:mf-f}. In \pref{app: odec}, besides proving \pref{thm: bilinear result}, we introduce the optimistic DEC notion, which is an extension from \cite{foster2024model} to the hybrid setting. The result in \pref{thm:mf-f} can thus also be generalized to learnable classes under optimistic DEC.


\begin{theorem}\label{thm: bilinear result}
For the hybrid bilinear class defined in \pref{def:adv bilinear}, \pref{alg:MAF} ensures with probability at least $1-\delta$:
\begin{itemize}
    \setlength{\itemsep}{0pt} 
    \setlength{\parskip}{0pt} 
    \setlength{\topsep}{0pt}
    \item If for all $\pi\in\Pi$, $\esttt(\pi) = \pi$, then $\Reg \le O\big(\sqrt{d\log\left(|\calF||\Pi|/\delta\right)}T^{\frac{3}{4}}\big)$.
    \item If there exists $\pi\in\Pi$ such that $\esttt(\pi) \neq \pi$, then $\Reg \le \order\big(\left(d\log\left(|\calF||\Pi|/\delta\right)\right)^{\frac{1}{3}}T^{\frac{5}{6}}\big)$.
\end{itemize}
\label{thm:mf-f}
\end{theorem}

\section{Discussions and Future Work}\label{sec: discussion}
In this work, we showed that the general framework developed in \pref{sec: general frame} is able to characterize the statistical complexity of hybrid MDPs with fixed transitions and adversarial rewards, and leads to near-optimal regret under model-based learning (\pref{sec: model-based}). 

The framework also extends to model-free value learning, as shown in \pref{sec: stochastic model-free} and \pref{sec: model-free adversarial}. However, additional challenges arise as it does not fully recover the result for the general bilinear class in the stochastic regime, and requires modifying the objective to an alternative DEC formulation in the hybrid regime. The latter poses a barrier to extending the framework to the model-free bandit feedback case. One future direction would be to more deeply investigate the role of the divergence measure in $\air^\Phi$ and see if there are better alternatives for model-free value learning. 

We believe that the framework has the potential to characterize the complexity of more settings that lie between model learning and policy learning.
Another avenue for future work is leveraging the framework we provide to derive efficient algorithms that do not have to maintain memory expensive priors or solve computationally complex saddle-point problems.


\bibliography{ref}

\appendix


\clearpage

\input{appendix-overview}

\input{appendix-stochastic}

\input{appendix-adversarial}
\input{appendix-support-lemma}

\end{document}

%% file: commands_colt.tex
\usepackage{xspace}
\usepackage{amsmath}
\usepackage{graphicx}
\usepackage{framed}
\usepackage{bbm}
\usepackage{algorithm}
\usepackage{natbib}
\usepackage{amssymb}
\usepackage{amssymb} 
\usepackage{amsmath} 
\usepackage{bm}
\usepackage{tabulary}
\usepackage{xcolor}
\usepackage{prettyref}
\usepackage{mathtools}
\newcommand{\cw}[1]{{\color{magenta}(CW: #1)}}
\newcommand{\jz}[1]{{\color{olive}JZ: #1}}
\usepackage{multirow}
\usepackage{nicefrac}

\usepackage{pifont}
\definecolor{Green}{rgb}{0.13, 0.65, 0.3}
\usepackage{colortbl}
\newcommand{\KL}{D_{\textsc{KL}}}
\newcommand{\sA}{\textsc{A}}

\DeclareMathOperator*{\argmin}{argmin} 
\DeclareMathOperator*{\argmax}{argmax} 

\usepackage{upgreek}

\newcommand{\Reg}{\text{\rmfamily Reg}}

\newcommand{\joint}{\Psi}

\newcommand{\bo}[1]{%
  \ifx#1\pi\boldsymbol{\uppi}
  \else\ifx#1\phi\boldsymbol{\phi}%
  \else\boldsymbol{#1}\fi%
}

\newcommand{\calA}{\mathcal{A}}

\newcommand{\calS}{\mathcal{S}}
\newcommand{\calM}{\mathcal{M}}

\newcommand{\calP}{\mathcal{P}}

\newcommand{\calF}{\mathcal{F}}
\newcommand{\calO}{\mathcal{O}}

\newcommand{\ellest}{\ell^{est}}

\newcommand{\calR}{\mathcal{R}}

\newcommand{\E}{\mathbb{E}}

\newcommand{\bbR}{\mathbb{R}}
\newcommand{\order}{O}

\newcommand{\inner}[1]{\left\langle#1\right\rangle}

\usepackage{tikz}

\newcommand{\nonl}{\renewcommand{\nl}{\let\nl}}

\usepackage{prettyref}
\newcommand{\pref}[1]{\prettyref{#1}}

\newcommand{\savehyperref}[2]{\texorpdfstring{\hyperref[#1]{#2}}{#2}}
\newrefformat{eq}{\savehyperref{#1}{Eq.~\textup{(\ref*{#1})}}}
\newrefformat{eqn}{\savehyperref{#1}{Equation~\ref*{#1}}}
\newrefformat{lem}{\savehyperref{#1}{Lemma~\ref*{#1}}}
\newrefformat{lemma}{\savehyperref{#1}{Lemma~\ref*{#1}}}
\newrefformat{def}{\savehyperref{#1}{Definition~\ref*{#1}}}
\newrefformat{line}{\savehyperref{#1}{Line~\ref*{#1}}}
\newrefformat{thm}{\savehyperref{#1}{Theorem~\ref*{#1}}}
\newrefformat{corr}{\savehyperref{#1}{Corollary~\ref*{#1}}}
\newrefformat{cor}{\savehyperref{#1}{Corollary~\ref*{#1}}}
\newrefformat{sec}{\savehyperref{#1}{Section~\ref*{#1}}}
\newrefformat{subsec}{\savehyperref{#1}{Section~\ref*{#1}}}
\newrefformat{app}{\savehyperref{#1}{Appendix~\ref*{#1}}}
\newrefformat{assum}{\savehyperref{#1}{Assumption~\ref*{#1}}}
\newrefformat{ex}{\savehyperref{#1}{Example~\ref*{#1}}}
\newrefformat{fig}{\savehyperref{#1}{Figure~\ref*{#1}}}
\newrefformat{alg}{\savehyperref{#1}{Algorithm~\ref*{#1}}}
\newrefformat{rem}{\savehyperref{#1}{Remark~\ref*{#1}}}
\newrefformat{conj}{\savehyperref{#1}{Conjecture~\ref*{#1}}}
\newrefformat{prop}{\savehyperref{#1}{Proposition~\ref*{#1}}}
\newrefformat{proto}{\savehyperref{#1}{Protocol~\ref*{#1}}}
\newrefformat{prob}{\savehyperref{#1}{Problem~\ref*{#1}}}
\newrefformat{claim}{\savehyperref{#1}{Claim~\ref*{#1}}}
\newrefformat{que}{\savehyperref{#1}{Question~\ref*{#1}}}
\newrefformat{op}{\savehyperref{#1}{Open Problem~\ref*{#1}}}
\newrefformat{fn}{\savehyperref{#1}{Footnote~\ref*{#1}}}
\newrefformat{tab}{\savehyperref{#1}{Table~\ref*{#1}}}
\newrefformat{fig}{\savehyperref{#1}{Figure~\ref*{#1}}}
\newrefformat{proc}{\savehyperref{#1}{Procedure~\ref*{#1}}}

\usepackage{caption}

\newcommand{\air}{\mathsf{AIR}}
\newcommand{\mair}{\mathsf{MAIR}}

\newcommand{\dec}{\mathsf{DEC}}

\newcommand{\co}{\textsc{co}}
\newcommand{\infoair}{\mathsf{InfoAIR}}
\newcommand{\odec}{\mathsf{ODEC}}

\renewcommand{\bar}[1]{\overline{#1}}
\renewcommand{\tilde}[1]{\widetilde{#1}}
\DeclareMathOperator*{\esttt}{est}
\DeclareMathOperator*{\unif}{unif}

%% file: appendix-overview.tex
\section{Proof of \pref{thm: general}}
\begin{proof}\textbf{of \pref{thm: general}}
    \begin{align*}
       &\E[\Reg(\pi_{1:T}^\star)] - \frac{\log|\Phi|}{\eta} \\
       &\leq \E\left[\sum_{t=1}^T \left(V_{M_t}(\pi_t^\star) -  V_{M_t}(\pi_t) - \frac{1}{\eta}\log\frac{\rho_{t+1}(\phi^\star)}{\rho_t(\phi^\star)} \right)\right] \\
       &= \E\left[\sum_{t=1}^T  \E_{\pi\sim p_t} \E_{o\sim M_t(\cdot|\pi)}\left[ V_{M_t}(\pi_t^\star) -  V_{M_t}(\pi) - \frac{1}{\eta} \log\frac{\nu_t(\phi^\star|\pi, o)}{\rho_t(\phi^\star)} \right]\right] \\
       \\&\leq \E\left[\sum_{t=1}^T  \max_{\mu \in \Delta(\Psi)}\E_{\phi^\star \sim \mu}\E_{(M, \pi^\star) \sim \mu(\cdot|\phi^\star)}\E_{\pi\sim p_t} \E_{o\sim M(\cdot|\pi)}\left[ V_{M}(\pi^\star) -  V_{M}(\pi) - \frac{1}{\eta} \log\frac{\nu_t(\phi^\star|\pi, o)}{\rho_t(\phi^\star)} \right]\right]
       \\&= \E\left[ \sum_{t=1}^T  \max_{\nu \in \Delta(\Psi)}\E_{\pi\sim p_t} \E_{(M,\pi^\star)\sim \nu} \E_{o\sim M(\cdot|\pi)}\left[ V_{M}(\pi^\star) -  V_{M}(\pi) - \frac{1}{\eta} \KL\left( \nu_{\bo{\phi}}(\cdot|\pi, o), \rho_t\right) \right]\right]  \tag{\pref{lem: max posterior}}   \\
       &= \E\left[ \sum_{t=1}^T  \min_{p\in\Delta(\Pi)}\max_{\nu\in\Delta(\joint)} \E_{\pi\sim p} \E_{(M,\pi^\star)\sim \nu} \E_{o\sim M(\cdot|\pi)}\left[ V_{M}(\pi^\star) -  V_{M}(\pi) - \frac{1}{\eta} \KL\left( \nu_{\bo{\phi}}(\cdot|\pi, o), \rho_t\right) \right]\right] \\
       &= \E\left[ \sum_{t=1}^T  \min_{p\in\Delta(\Pi)}\max_{\nu\in\Delta(\joint)}  
\air^\Phi_{\rho_t,\eta}(p,\nu) \right] \\
       &= \E\left[ \sum_{t=1}^T  \max_{\nu\in\Delta(\joint)}\min_{p\in\Delta(\Pi)}  
\air^\Phi_{\rho_t,\eta}(p,\nu) \right]   \tag{$\air^\Phi_{\rho_t,\eta}(p,\nu)$ is convex in $p$ and concave in $\nu$ and \pref{lem:minmax}}  \\
    &\leq  T \max_{\rho \in \Delta(\Phi)}\max_{\nu\in\Delta(\joint)}\min_{p\in\Delta(\Pi)}  
\air^\Phi_{\rho,\eta}(p,\nu). 
    \end{align*}
\end{proof}

\section{Relation with DEC}
We introduce a useful notion of a $\dec$ like complexity that equivalently describes the complexity of $\air^{\Phi}$.

\begin{definition} The $\Phi$-dependent DEC is defined as
\begin{align*}
&\dec_{\eta}^{\textsc{KL}}\left(\calM, \Phi\right) = 
\\&\max_{\Bar{M} \in \textsc{co}(\calM)}\min_{p\in\Delta(\Pi)}\max_{\phi\in\Phi} \max_{\nu\in \Delta(\phi)}   \E_{\pi\sim p} \E_{\phi\sim \nu}   \left[\E_{(M,\pi^\star)\sim \nu}[V_{M}(\pi^\star)] - V_{M^\nu}(\pi) - \frac{1}{\eta}  \KL\left( M^{\nu}(\cdot|\pi), \Bar{M}(\cdot|\pi) \right)\right]. 
\end{align*}
\end{definition}

We first establish that the maxmin of $\air^\Phi$ is bounded by the new notion of DEC.
\begin{lemma} \label{lem: equivalence of dec}$\displaystyle\max_{\rho \in\Delta(\Phi)}\max_{\nu\in\Delta(\joint)}\min_{p\in\Delta(\Pi)} \air^\Phi_{\rho,\eta}(p,\nu) = \dec^{\textsc{KL}}_\eta (\calM, \Phi)$. 
%
\end{lemma}
\begin{proof}\textbf{of \pref{lem: equivalence of dec}\ \ }
\begin{align*}
    \air^\Phi_{\rho,\eta}(p,\nu) &= \E_{\pi\sim p} \E_{(M,\pi^\star)\sim \nu} \E_{o\sim M(\cdot|\pi)}\left[ V_{M}(\pi^\star) -  V_{M}(\pi) - \frac{1}{\eta} \KL\left( \nu_{\bo{\phi}}(\cdot|\pi, o), \rho\right) \right] 
        \\
        &= \E_{\pi\sim p} \E_{(M,\pi^\star)\sim \nu} \E_{o\sim M(\cdot|\pi)}\left[ V_{M}(\pi^\star) -  V_{M}(\pi) - \frac{1}{\eta} \KL\left( \nu_{\bo{\phi}}(\cdot|\pi, o), \nu_{\bo{\phi}}\right) - \frac{1}{\eta}\KL\left( \nu_{\bo{\phi}}, \rho\right)\right]
\end{align*}
is jointly concave over $(\nu,\rho)$, hence by \pref{lem:minmax}
\begin{align*}
    &\max_{\rho \in \Delta(\Phi)}\max_{\nu\in \Delta(\Psi)} \min_{p\in\Delta(\Pi)}\air^\Phi_{\rho,\eta}(p,\nu)
    =\min_{p\in\Delta(\Pi)}\max_{\nu\in \Delta(\Psi)}\max_{\rho \in \Delta(\Phi)} \air^\Phi_{\rho,\eta}(p,\nu)\,.
\end{align*}
Let's assume $\nu$ and $\pi$ are fixed, we have for the $\rho$-dependent part
\begin{align*}
    \max_{\rho \in \Delta(\Phi)} \E_{(M,\pi^\star)\sim \nu} \E_{o\sim M(\cdot|\pi)}\left[-\KL\left( \nu_{\bo{\phi}}, \rho\right)\right]=0\,,
\end{align*}
with the argmax at $\rho=\nu_{\bo{\phi}}$, hence
\begin{align*}
    &\max_{\rho \in \Delta(\Phi)} \air^\Phi_{\rho,\eta}(p,\nu)=\air^\Phi_{\nu_{\bo{\phi}},\eta}(p,\nu)\,.
\end{align*}
We continue with the KL term in this expression. In an overload of notation, let $\nu$ conditioned on $\pi$ define a distribution over $\Delta(\Psi)\times\calO$ via $\nu(o,M,\pi^\star|\pi)=\nu(M,\pi^\star)M(o|\pi)$.
    Note that for fixed $\pi$, 
    \begin{align*}
        \E_{(M,\pi^\star)\sim \nu}\E_{o\sim M(\cdot|\pi)} \KL(\nu_{\bo{\phi}}(\cdot|\pi, o), \nu_{\bo{\phi}})  
        &=\E_{o\sim \nu_{\bo{o}}(\cdot|\pi)} \KL(\nu_{\bo{\phi}}(\cdot|\pi, o), \nu_{\bo{\phi}})  \\
        &=\E_{o\sim \nu_{\bo{o}}(\cdot|\pi)}\E_{\phi\sim\nu_{\bo{\phi}}(\cdot|\pi, o)} \log\left(\frac{\nu(\phi|\pi, o)}{\nu(\phi)}\right)  \\
        &=\E_{\phi\sim \nu_{\bo{\phi}}}\E_{o\sim\nu_{\bo{o}}(\cdot|\phi,\pi)} \log\left(\frac{\nu(\phi|\pi, o)}{\nu(\phi)}\right)  \\
        &=\E_{\phi\sim \nu_{\bo{\phi}}}\E_{o\sim\nu_{\bo{o}}(\cdot|\phi,\pi)} \log\left(\frac{\nu(o|\phi,\pi)}{\nu(o|\pi)}\right)  \tag{Bayes-rule}\\
        &= \E_{\phi\sim \nu} \KL\left( \nu_{\bo{o}}(\cdot|\phi,\pi),  \nu_{\bo{o}}(\cdot|\pi) \right)\,.
    \end{align*}
    We now consider $\mu\in\Delta(\calM)$, which again is overloaded to denote a distribution over $\Delta(\calM)\times\calO$.
    \begin{align*}
    -\E_{\phi\sim \nu} \KL\left( \nu_{\bo{o}}(\cdot|\phi,\pi),  \nu_{\bo{o}}(\cdot|\pi) \right) &= \max_{\mu\in\Delta(\calM)} \E_{\phi\sim \nu} \left[-\KL\left( \nu_{\bo{o}}(\cdot|\phi,\pi),  \nu_{\bo{o}}(\cdot|\pi) \right)-\KL\left( \nu_{\bo{o}}(\cdot|\pi),\mu_{\bo{o}}(\cdot|\pi) \right)\right]\\
    &=\max_{\mu\in\Delta(\calM)} \E_{\phi\sim \nu} \left[-\KL\left( \nu_{\bo{o}}(\cdot|\phi,\pi),  \mu_{\bo{o}}(\cdot|\pi) \right)\right]
    \end{align*}
Combining everything, we have shown that
\begin{align*}
    &\max_{\rho \in \Delta(\Phi)}\max_{\nu\in \Delta(\Psi)} \min_{p\in\Delta(\Pi)}\air^\Phi_{\rho,\eta}(p,\nu) \\
    &= \min_{p\in\Delta(\Pi)} \max_{\nu\in \Delta(\Psi)}\max_{\mu\in\Delta(\calM)}\E_{\pi\sim p} \E_{(M,\pi^\star)\sim \nu} \E_{o\sim M(\cdot|\pi)}\left[ V_{M}(\pi^\star) -  V_{M}(\pi) - \frac{1}{\eta}\KL\left( \nu_{\bo{o}}(\cdot|\phi,\pi),  \mu_{\bo{o}}(\cdot|\pi) \right)\right]\,.
\end{align*}
Again this is jointly concave over $(\nu,\mu)$ and convex concave in $(p,\nu)$ given $\mu$, hence by \pref{lem:minmax}
\begin{align*}
    &= \max_{\mu\in\Delta(\calM)}\min_{p\in\Delta(\Pi)} \max_{\nu\in \Delta(\Psi)}\E_{\pi\sim p} \E_{(M,\pi^\star)\sim \nu} \E_{o\sim M(\cdot|\pi)}\left[ V_{M}(\pi^\star) -  V_{M}(\pi) - \frac{1}{\eta}\KL\left( \nu_{\bo{o}}(\cdot|\phi,\pi),  \mu_{\bo{o}}(\cdot|\pi) \right)\right]\\
    &= \max_{\bar M\in\co(\calM)}\min_{p\in\Delta(\Pi)} \max_{\nu\in \Delta(\Psi)}\E_{\pi\sim p} \E_{(M,\pi^\star)\sim \nu} \E_{o\sim M(\cdot|\pi)}\left[ V_{M}(\pi^\star) -  V_{M}(\pi) - \frac{1}{\eta}\KL\left( M^{\nu(\cdot|\phi)}(\cdot|\pi),  \bar M(\cdot|\pi) \right)\right]\\
    &= \dec_{\eta}^{\textsc{KL}}\left(\calM, \Phi\right)\,.
\end{align*}
\end{proof}

\begin{proof}\textbf{of \pref{lem: air by dec}\ \ }
By \pref{lem: equivalence of dec}, we only need to compare $\dec^{\textrm{KL}}_\eta(\calM, \Phi)$ and $\dec^{\textrm{KL}}_\eta(\bar{\calM}(\Phi))$. 
Define 
    \begin{align*}
        C(\Phi) := \max_{\phi\in\Phi}\max_{\nu\in\Delta(\phi)}\E_{(M,\pi^\star)\sim\nu}[V_M(\pi^\star)]-V_{M^\nu}(\pi_{M^\nu})\,,
    \end{align*}
where $M^\nu$ is the mixture model $M^\nu(\cdot|\pi) := \E_{M\sim \nu}[M(\cdot|\pi)]$, and $\pi_{M^\nu} := \argmax_{\pi\in\Pi} V_{M^\nu}(\pi)$. 
In the fixed-comparator case, we have
\begin{align*}
&\max_{\phi\in\Phi}\max_{\nu\in\Delta(\phi)}\E_{(M,\pi^\star)\sim\nu}[V_M(\pi^\star)]-V_{M^\nu}(\pi_{M^\nu})\\
&=\max_{\phi\in\Phi}\max_{\nu\in\Delta(\phi)}V_{M^\nu}(\pi^\phi)-V_{M^\nu}(\pi_{M^\nu})\leq 0\,.
\end{align*}
We have
\begin{align*}
    &\max_{\phi\in\Phi} \max_{\nu\in \Delta(\phi)}   \E_{\pi\sim p} \E_{\phi\sim \nu}   \left[\E_{(M,\pi^\star)\sim \nu}[V_{M}(\pi^\star)] - V_{M^\nu}(\pi) - \frac{1}{\eta}  \KL\left( M^{\nu}(\cdot|\pi), \Bar{M}(\cdot|\pi) \right)\right]\\
    &\leq \max_{\phi\in\Phi} \max_{\nu\in \Delta(\phi)}   \E_{\pi\sim p} \E_{\phi\sim \nu}   \left[V_{M^\nu}(\pi_{M^\nu}) - V_{M^\nu}(\pi) - \frac{1}{\eta}  \KL\left( M^{\nu}(\cdot|\pi), \Bar{M}(\cdot|\pi) \right)\right] + C(\Phi)\\
    &=\max_{M^\nu\in\bar\calM(\Phi)}  \E_{\pi\sim p} \E_{\phi\sim \nu}   \left[V_{M^\nu}(\pi_{M^\nu}) - V_{M^\nu}(\pi) - \frac{1}{\eta}  \KL\left( M^{\nu}(\cdot|\pi), \Bar{M}(\cdot|\pi) \right)\right] + C(\Phi)
\end{align*}
Hence
\begin{align*}
    \dec^{\textsc{kl}}_\eta(\calM,\Phi)\leq \dec^{\textsc{kl}}_\eta(\bar\calM(\Phi)) + C(\Phi)\,. 
\end{align*}
If $C(\Phi)\leq 0$, which includes all fixed-comparator environments, this implies the Lemma.

Finally, assume $C(\Phi)>0$. We show that these environments are unlearnable and that $\dec^{\textsc{kl}}_\eta(\calM,\Phi)=\Omega(1)$ independent of $\eta$.
Fix $(\phi', \nu') = \argmax_{\phi\in\Phi,\nu\in\Delta(\phi)}\E_{(M,\pi^\star)\sim\nu}[V_M(\pi^\star)]-V_{M^\nu}(\pi_{M^\nu})$.
\begin{align*}
    &\dec^{\textsc{kl}}_\eta(\calM,\Phi)\\
    &\geq \min_{p\in\Delta(\Pi)}\max_{\phi\in\Phi} \max_{\nu\in \Delta(\phi)}   \E_{\pi\sim p} \E_{\phi\sim \nu}   \left[\E_{(M,\pi^\star)\sim \nu}[V_{M}(\pi^\star)] - V_{M^\nu}(\pi) - \frac{1}{\eta}  \KL\left( M^{\nu}(\cdot|\pi), M^{\nu'}(\cdot|\pi) \right)\right]\\
    &\geq \min_{p\in\Delta(\Pi)}\E_{\pi\sim p} \left[\E_{(M,\pi^\star)\sim \nu'}[V_{M}(\pi^\star)] - V_{M^{\nu'}}(\pi) - \frac{1}{\eta}  \KL\left( M^{\nu'}(\cdot|\pi), M^{\nu'}(\cdot|\pi) \right)\right]\\
    &=C(\Phi)=\Omega(1)\,.
\end{align*}
To show that the problem is indeed unlearnable, let the environment sample i.i.d. $(M_t,\pi^\star_t)\sim\nu'$ in every round. For any policy of the agent, the expected regret in every round is at least $C(\Phi)$, so the expected regret is lower bounded by $C(\Phi)T$.

\end{proof}
\begin{proof}\textbf{of \pref{lem: dec equiv}\ \ }
Using \pref{lem: equivalence of dec}, we only need to compare the two notions of $\dec$.
Recall we consider the case where $\Phi=\Theta\times\Pi$.
    We have for fixed $p$ and $\bar M$
    \begin{align*}
    &\max_{\nu\in \Delta(\Psi)}   \E_{\pi\sim p} \E_{\phi\sim \nu}  \E_{(M,\pi^\star)\sim \nu(\cdot|\phi)} \left[V_{M}(\pi^\star) - V_{M}(\pi) - \frac{1}{\eta}  \KL\left( M^{\nu(\cdot|\phi)}(\cdot|\pi), \Bar{M}(\cdot|\pi) \right)\right]\\
    &=\max_{\theta\in\Theta}\max_{\nu\in \Delta(\theta)}\max_{\mu\in\Delta(\Pi)}   \E_{\pi\sim p}  \E_{(M,\pi^\star)\sim (\nu,\mu)} \left[V_{M}(\pi^\star) - V_{M}(\pi) - \frac{1}{\eta}  \KL\left( M^{\nu}(\cdot|\pi), \Bar{M}(\cdot|\pi) \right)\right]\\
    &=\max_{\theta\in\Theta}\max_{\nu\in \Delta(\theta)}\E_{\pi\sim p}  \E_{M\sim \nu} \left[V_{M}(\pi_{M^\nu}) - V_{M}(\pi) - \frac{1}{\eta}  \KL\left( M^{\nu}(\cdot|\pi), \Bar{M}(\cdot|\pi) \right)\right]\\
    &=\max_{M^\nu\in\bar\calM(\Phi)}\E_{\pi\sim p}  \left[V_{M^\nu}(\pi_{M^\nu}) - V_{M^\nu}(\pi) - \frac{1}{\eta}  \KL\left( M^{\nu}(\cdot|\pi), \Bar{M}(\cdot|\pi) \right)\right]\,.
    \end{align*}
    Since the first and last line are equal, taking the $\max$ over $\bar M$ and $\min$ over $p$ is still equal, which shows the Lemma.
\end{proof}

\begin{lemma}\label{lem: trimming}
Let $\Phi$ describe a trimmed and aggregated stochastic game, i.e. for all $\phi\in\Theta$ we have that $\phi=\{(M,\phi^\star_\theta) | M\in \calM_\phi\}$ and such that $\pi_{M}=\pi_{\theta}$ for all $M\in\calM_{\phi}$.

Further denote $\bar\Phi$ the unaggregated version, 
$\bar\Phi := \{\{(M,\pi)|M\in\calM_{\phi}\}|\phi\in\Phi,\pi\in\Pi\}$.
Then
$\dec_{\eta}^{\textsc{kl}}(\Phi)=\dec_{\eta}^{\textsc{kl}}(\bar\Phi)$ and $\bar\calM(\Phi)=\bar\calM(\bar\Phi)$.
\end{lemma}
\begin{proof}
The second statement is trivial since the convexivication is policy-unaware.
For the first statement, observe that for any function $F:\bigcup_{\phi\in\Phi}\Delta(\calM_{\phi})\times\Pi\rightarrow \bbR$, such that for any $\pi^\star_{\phi}=\argmax_{\pi\in\Pi}F(\nu_\phi,\pi)$ for any $\nu_\phi\in\Delta(\phi)$, we have
\begin{align*}
    \max_{(\calM_\phi,\pi)\in\bar\Phi}\max_{\nu_\phi\in\Delta(\calM_\phi)}F(\nu_{\phi},\pi)&=\max_{\phi\in\Phi}\max_{\nu_\phi\in\Delta(\calM_\phi)}\max_{\pi\in\Pi}F(\nu_{\phi},\pi)\\
    &=\max_{\phi\in\Phi}\max_{\nu_\phi\in\Delta(\calM_\phi)}F(\nu_{\phi},\pi^\star_\phi)\,.
\end{align*}
Finally, notice that $\dec_{\eta}^{\textsc{kl}}(\calM,\Phi)$ for a fixed $\bar M$ and $\pi$ is exactly such a function $F$. Using the equality above shows the equality between the complexities.
\end{proof}

\begin{lemma}\label{lem: max infoair to air}
    For fixed comparator games, i.e. $\Phi=\Theta\times\Pi$, we have 
$$\max_{\nu\in\Delta(\calM)}\max_{\pi^\Theta\in\Pi^\Theta}\infoair^\Theta_{\rho_t,\eta}(p,\nu,\pi_t^\Theta)
=\max_{\nu\in\Delta(\Psi)}\air^{\Phi}_{\rho,\eta}(p,\nu)\,.$$
\end{lemma}
\begin{proof}
    The RHS can be seen as a function 
    \begin{align}
        \max_{\nu\in\Delta(\Psi)}\E_{\phi^\star\sim\nu}\left[F(\nu_{\bo{M}}(\cdot|\phi^\star),\pi^{\phi^\star};\rho,\eta,p)\right]\,.  \label{eq: obje}
    \end{align}
    For any $\theta\in\Theta$, define 
    \begin{align*}
        \phi^\star(\theta) := \argmax_{\phi\in \theta\times\Pi}F(\nu_{\bo{M}}(\cdot|\phi),\pi^{\phi};\rho,\eta,p)\,.
    \end{align*}
    The objective \pref{eq: obje} is maximized when $\forall\phi\in\Phi$: 
    $\nu_{\bo{M}}(\cdot|\phi)=\argmax_{\mu\in\Delta(\phi)}F(\mu_{\bo{M}},\pi^{\phi};\rho,\eta,p)$.
    Also, we never decrease the objective by shifting probability mass from $\phi=(\theta,\pi)$ towards $\phi^\star(\theta)$ by the definition of $\phi^\star(\theta)$.
    Hence 
    \begin{align*}
        \max_{\nu\in\Delta(\Psi)}\air^{\Phi}_{\rho,\eta}(p,\nu
        )&=\max_{\nu\in\Delta(\Psi)}\E_{\phi\sim\nu}\left[F(\nu_{\bo{M}}(\cdot|\phi),\pi^{\phi};\rho,\eta,p)\right]
        \\&=\max_{\mu\in\Delta(\calM)}\E_{\theta\sim\mu}\left[F(\mu(\cdot|\theta),\pi^{\phi^\star(\theta)};\rho,\eta,p)\right]
        \\&=\max_{\mu\in\Delta(\calM)}\max_{\pi^\Theta\in\Pi^\Theta}\E_{\theta\sim\mu}\left[F(\mu(\cdot|\theta),\pi^{\theta};\rho,\eta,p)\right]
        \\&=\max_{\nu\in\Delta(\calM)}\max_{\pi^\Theta\in\Pi^\Theta}\infoair^\Theta_{\rho_t,\eta}(p,\nu,\pi^\Theta)\,.
    \end{align*}
\end{proof}
\begin{proof}\textbf{of \pref{thm: informed}\ \ }
The proof is almost the same as \pref{thm: general} but changing $\pi^\star$ to the revealed $\pi^\theta$ and not requiring a joint distribution over $\Theta\times\Pi$. We give a full proof here for completion. 
    \begin{align*}
    &\E[\Reg(\pi_{1:T}^{\theta^\star})] - \frac{\log|\Theta|}{\eta} \\
       &\leq \E\left[\sum_{t=1}^T \left(V_{M_t}(\pi_t^{\theta^\star}) -  V_{M_t}(\pi_t) - \frac{1}{\eta}\log\frac{\rho_{t+1}(\theta^\star)}{\rho_t(\theta^\star)} \right)\right] \\
       &= \E\left[\sum_{t=1}^T  \E_{\pi\sim p_t} \E_{o\sim M_t(\cdot|\pi)}\left[ V_{M_t}(\pi_t^{\theta^\star}) -  V_{M_t}(\pi) - \frac{1}{\eta} \log\frac{\nu_t(\theta^\star|\pi, o)}{\rho_t(\theta^\star)} \right]\right] \\
       \\&\le \E\left[\sum_{t=1}^T  \max_{\mu \in \Delta(\calM)}\E_{\theta^\star \sim \mu}\E_{M \sim \mu(\cdot|\theta^\star)}\E_{\pi\sim p_t} \E_{o\sim M(\cdot|\pi)}\left[ V_{M}(\pi_t^{\theta^\star}) -  V_{M}(\pi) - \frac{1}{\eta} \log\frac{\nu_t(\theta^\star|\pi, o)}{\rho_t(\theta^\star)} \right]\right] 
       \\&= \E\left[ \sum_{t=1}^T  \max_{\nu\in\Delta(\calM)} \E_{\pi\sim p_t} \E_{\theta \sim \nu}\E_{M\sim \nu(\cdot|\theta)} \E_{o\sim M(\cdot|\pi)}\left[ V_{M}(\pi_t^{\theta}) -  V_{M}(\pi) - \frac{1}{\eta} \KL\left( \nu_{\bo{\theta}}(\cdot|\pi, o), \rho_t\right) \right]\right] \tag{\pref{lem: max posterior}}  \\
       &= \E\left[ \sum_{t=1}^T  \min_{p\in\Delta(\Pi)}\max_{\nu\in\Delta(\calM)} \E_{\pi\sim p} \E_{\theta \sim \nu}\E_{M\sim \nu(\cdot|\theta)} \E_{o\sim M(\cdot|\pi)}\left[ V_{M}(\pi_t^{\theta}) -  V_{M}(\pi) - \frac{1}{\eta} \KL\left( \nu_{\bo{\theta}}(\cdot|\pi, o), \rho_t\right) \right]\right] \\
       &= \E\left[ \sum_{t=1}^T  \min_{p\in\Delta(\Pi)}\max_{\nu\in\Delta(\calM)}  
 \infoair^\Theta{\rho_t,\eta}(p,\nu,\pi_t^\Theta)\right] \\
       &= \E\left[ \sum_{t=1}^T  \max_{\nu\in\Delta(\calM)}\min_{p\in\Delta(\Pi)}  
 \infoair^\Theta{\rho_t,\eta}(p,\nu,\pi_t^\Theta)\right]   \tag{\pref{lem:minmax}}  \\
    &\leq  T \max_{\pi^\Theta\in\Pi^\Theta}\max_{\rho\in \Delta(\Theta)}\max_{\nu\in\Delta(\calM)}\min_{p\in\Delta(\Pi)}  
\infoair^\Theta_{\rho,\eta}(p,\nu,\pi^\theta). 
\end{align*}
Finally, by \pref{lem:minmax} and \pref{lem: max infoair to air}, we have
\begin{align*}
    \max_{\pi^\Theta\in\Pi^\Theta}\max_{\rho\in \Delta(\Theta)}\max_{\nu\in\Delta(\calM)}\min_{p\in\Delta(\Pi)}  
\infoair^\Theta_{\rho,\eta}(p,\nu,\pi^\theta) = \max_{\rho\in \Delta(\Phi)}\max_{\nu\in\Delta(\joint)}\min_{p\in\Delta(\Pi)}  
\air^\Phi_{\rho,\eta}(p,\nu)\,.
\end{align*}
Using \pref{lem: dec equiv} and noting $\bar\calM(\Phi)=\bar\calM(\Theta)$ finishes the proof.
\end{proof}

\begin{proof}\textbf{of \pref{thm: simultaneous full-info}\ \ }
Define the players (randomized) meta-policy as
\begin{align*}
    \pi_t^\theta(a|s) \propto \exp\left(\gamma\sum_{i=1}^{t-1}Q^{\pi_i^\theta}(s,a;M^\theta_i)\right)\,,
\end{align*}
where $Q^{\pi}(\cdot,\cdot;M)$ is the $Q$ function of policy $\pi$ under model $M$. That this policy is computable relies crucially on the availability of $M_t^{\theta}$ for all $\theta\in\Theta$.
 We have

\begin{align*}
&\E\left[\sum_{t=1}^T \left(V_{M_t}(\pi^\star) - V_{M_t}(\pi_t^{\theta^\star})\right)\right]
\\&\le \E\left[\E_{s}^{M_t^{\theta^\star},\pi^\star} \left[\sum_{t=1}^T \sum_{a}\left(\pi^\star(a|s) - \pi_t^{\theta^\star}(a|s)\right)Q^{\pi_t^{\theta^\star}}(s, a; M_t^{\theta^\star})\right]\right]  \tag{\pref{lem:PDL}}
\\&\le \E\left[\E_{s}^{M_t^{\theta^\star},\pi^\star} \left[\frac{\log|\calA|}{\gamma}+\gamma T\right]\right]  \tag{$Q^\pi(s,a;M) \le 1$ and \pref{lem:EXP}}
\\&\le \left(\frac{\log|\calA|}{\gamma} + \gamma T\right) H \tag{$H$-stage MDP}\,.   
\end{align*}
\end{proof}

%% file: appendix-stochastic.tex
\section{Omitted Proofs in \pref{sec: stochastic model-free}}
\begin{proof}\textbf{of \pref{thm: linear QV}\ \ }
Notice that this is a fixed comparator game where every $\phi$ has a single policy. By \pref{lem: dec equiv} and \pref{lem: equivalence of dec}, it suffices to bound 
\begin{align*}
    \dec^{\textsc{KL}}_{\eta}(\bar\calM(\Phi)) = \dec^{\textsc{KL}}_\eta(\calM,\Phi).  
\end{align*}
    Let $\pi^\phi, \theta^\phi, w^\phi$ denote the $\pi, \theta, w$ for the group $\phi\in\Phi$.  Let $p$ be the marginal distribution of $\nu$. Then
    \begin{align}
        &\mathbb{E}_{\pi\sim p}  \E_{(M,\pi^\star)\sim \nu}\left[ V_M(\pi^\star) - V_M(\pi)\right]  \nonumber  \\
        &= \E_{(M,\pi)\sim \nu} \E_{(M',\pi')\sim \nu}\left[ V_M(\pi) - V_{M'}(\pi)\right] \nonumber \\
        &= \E_{\phi\sim \nu}\E_{(M,\pi)\sim \nu(\cdot|\phi)} \E_{(M', \pi')\sim \nu}\left[ V_M(\pi^\phi) - V_{M'}(\pi^\phi)\right] \nonumber \\
        &= \E_{\phi\sim \nu}\E_{(M,\pi)\sim \nu(\cdot|\phi)} \E_{(M', \pi')\sim \nu}\E^{\pi^\phi, M'}\left[ \sum_{h=1}^H \left(Q_h(s_h,a_h; M) - r_h - V_{h+1}(s_{h+1}; M)\right) \right] \nonumber \\
        &= \E_{\phi\sim \nu}\E_{\phi'\sim \nu}\E_{(M,\pi)\sim \nu(\cdot|\phi)} \E_{(M', \pi')\sim \nu(\cdot|\phi')}\E^{\pi^\phi,M'}\Bigg[ \sum_{h=1}^H (Q_h(s_h,a_h; M) - V_{h+1}(s_{h+1}; M) \nonumber \\
        &\qquad \qquad \qquad \qquad \qquad \qquad \qquad \qquad \qquad \qquad \qquad - Q_h(s_h,a_h; M') + V_{h+1}(s_{h+1}; M')) \Bigg] \nonumber \\
        &= \E_{\phi\sim \nu}\E_{\phi'\sim \nu}\E_{(M', \pi')\sim \nu(\cdot|\phi')}\E^{\pi^\phi,M'}\left[ \sum_{h=1}^H \inner{\varphi(s_h,a_h), \theta^\phi - \theta^{\phi'}} - \inner{\psi(s_{h+1}), w^\phi - w^{\phi'}}\right]\nonumber  \\
        &= \E_{\phi\sim \nu} \E_{\phi'\sim \nu} \E_{(M', \pi')\sim \nu(\cdot|\phi')} \sum_{h=1}^H 
 \inner{X_h(\pi^\phi; M'), W(\phi) - W(\phi')}   \label{eq: break1}
    \end{align}
    where we define for all $h$ 
    \begin{align*}
        X_h(\phi; M') &= \mathbb{E}^{\pi^\phi, M'} \left[\varphi(s_h,a_h), -\psi(s_{h+1})\right], \\
        W(\phi) &= \left[ \theta^\phi, w^\phi \right]. 
    \end{align*}
    Let $M^\mu$ denote the mixture model such that $M^\mu(\cdot|\pi) = \E_{M\sim \mu}[M(\cdot|\pi)]$. Furthermore, in this proof, we denote $M^{\nu(\cdot|\phi)}$ as $M^\phi$.  Then the last expression in \pref{eq: break1} can be written as 
    \begin{align}
        \sum_{h=1}^H  \E_{\phi\sim \nu} \E_{\phi'\sim \nu} \inner{X_h(\phi; M^{\nu(\cdot|\phi')}), W(\phi) - W(\phi') } = \sum_{h=1}^H  \E_{\phi\sim \nu} \E_{\phi'\sim \nu} \inner{X_h(\phi; M^{\phi'}), W(\phi) - W(\phi') }.  \label{eq: break 4} 
    \end{align}
    For each $h$ and $\phi'$, define 
    \begin{align*}
        \Sigma_h^{\phi'} = \E_{\phi\sim \nu}\left[X_h(\phi; M^{\phi'})X_h(\phi; M^{\phi'
        })^\top\right].  
    \end{align*}
    Then we have 
    \begin{align}
        &\E_{\phi\sim \nu} \E_{\phi'\sim \nu} \inner{X_h(\phi; M^{\phi'}), W(\phi) - W(\phi') }   \nonumber \\
        &\leq \E_{\phi'\sim \nu} \left[\sqrt{\E_{\phi\sim \nu}\left[\left\|X_h(\phi; M^{\phi'})\right\|_{(\Sigma^{\phi'}_h)^{-1}}^2\right]}\sqrt{\E_{\phi\sim \nu} \left[\left\|W(\phi)-W(\phi')\right\|_{\Sigma^{\phi'}_h}^2\right]}\right] \nonumber 
  \\
        &= \E_{\phi'\sim \nu}\left[\sqrt{2d} \sqrt{\E_{\phi\sim \nu}  \E_{\tilde{\phi}\sim \nu} \left[\inner{X_h(\tilde{\phi};M^{\phi'}), W(\phi) - W(\phi')}^2\right]  }\right].  \label{eq: break 2}
    \end{align}
    We continue to bound the square term in \pref{eq: break 2}: 
    \begin{align*}
        &\E_{\phi\sim \nu}  \E_{\tilde{\phi}\sim \nu} \left[\inner{X_h(\tilde{\phi};M^{\phi'}), W(\phi) - W(\phi')}^2\right] \\
        &= \E_{\phi\sim \nu}  \E_{\tilde{\phi}\sim \nu} \left[\left(\E^{\pi^{\tilde{\phi}}, M^{\phi'}} 
\left[ \varphi(s_h,a_h)^\top \theta^\phi -\psi(s_{h+1})^\top w^{\phi} \right] - \E^{\pi^{\tilde{\phi}}, M^{\phi'}} 
\left[ \varphi(s_h,a_h)^\top \theta^{\phi'} -\psi(s_{h+1})^\top w^{\phi'} \right] \right)^2\right] \\
&= \E_{\phi\sim \nu}  \E_{\tilde{\phi}\sim \nu} \left[\left(\E^{\pi^{\tilde{\phi}}, M^{\phi'}} 
\left[ \varphi(s_h,a_h)^\top \theta^\phi- r_h -  \psi(s_{h+1})^\top w^{\phi} \right] 
\right)^2\right] \\
&=  \E_{\phi\sim \nu}  \E_{\tilde{\phi}\sim \nu} \Bigg[\Bigg(\E^{\pi^{\tilde{\phi}}, M^{\phi'}} 
\left[ \varphi(s_h,a_h)^\top \theta^\phi- r_h -  \psi(s_{h+1})^\top w^{\phi} \right] \\
&\qquad \qquad \qquad \qquad \qquad \qquad - \underbrace{\E^{\pi^{\tilde{\phi}}, M^{\phi}} 
\left[ \varphi(s_h,a_h)^\top \theta^\phi- r_h -  \psi(s_{h+1})^\top w^{\phi} \right]}_{=0}
\Bigg)^2\Bigg] \\
&\leq \E_{\phi\sim \nu}  \E_{\tilde{\phi}\sim \nu} 
 \left[D_{\textsc{TV}}^2 \left(M^{\phi'}(\cdot|\pi^{\tilde{\phi}}), M^{\phi}(\cdot|\pi^{\tilde{\phi}})\right)\right] \\
 &\leq 2\E_{\phi\sim \nu}  \E_{\tilde{\phi}\sim \nu} 
 \left[D_{\textsc{H}}^2 \left(M^{\phi'}(\cdot|\pi^{\tilde{\phi}}), M^{\phi}(\cdot|\pi^{\tilde{\phi}})\right)\right]. 
    \end{align*}
Combining this with \pref{eq: break 2}, we get 
\begin{align*}
   &\E_{\phi\sim \nu} \E_{\phi'\sim \nu} \inner{X_h(\phi; M^{\phi'}), W(\phi) - W(\phi') } \\
   &\leq 2\sqrt{d \E_{\phi'\sim \nu}\E_{\phi\sim \nu}  \E_{\tilde{\phi}\sim \nu} 
 \left[D_{\textsc{H}}^2 \left(M^{\phi'}(\cdot|\pi^{\tilde{\phi}}), M^{\phi}(\cdot|\pi^{\tilde{\phi}})\right)\right]} \\
 &\leq 2\sqrt{2d \E_{\phi'\sim \nu}\E_{\phi\sim \nu}  \E_{\tilde{\phi}\sim \nu} 
 \left[D_{\textsc{H}}^2 \left(M^{\nu(\cdot|\phi')}(\cdot|\pi^{\tilde{\phi}}), M^{\nu}(\cdot|\pi^{\tilde{\phi}})\right) + D_{\textsc{H}}^2 \left(M^{\nu}(\cdot|\pi^{\tilde{\phi}}), M^{\nu(\cdot|\phi)}(\cdot|\pi^{\tilde{\phi}})\right)\right]} \\
 &\leq 4\sqrt{d \E_{\phi\sim \nu}  \E_{\tilde{\phi}\sim \nu} 
 \left[D_{\textsc{H}}^2 \left(M^{\nu(\cdot|\phi)}(\cdot|\pi^{\tilde{\phi}}), M^{\nu}(\cdot|\pi^{\tilde{\phi}})\right)\right]} \\
 &= 4\sqrt{d \E_{\pi\sim p}\E_{\phi\sim \nu} 
 \left[\KL \left(M^{\nu(\cdot|\phi)}(\cdot|\pi), M^{\nu}(\cdot|\pi)\right)\right]}. 
\end{align*}
Further combining this with \pref{eq: break 4} and \pref{eq: break1}, we get 
\begin{align*}
    &\mathbb{E}_{\pi\sim p}  \E_{\phi\sim \nu}\left[ V_{M^{\nu(\cdot|\phi)}}(\pi^\phi) - V_{M^{\nu(\cdot|\phi)}}(\pi) - \frac{1}{\eta}\KL\left(M^{\nu(\cdot|\phi)}(\cdot|\pi), M^{\nu}(\cdot|\pi)\right)\right] \\ 
    &=\mathbb{E}_{\pi\sim p}  \E_{(M,\pi^\star)\sim \nu}\left[ V_M(\pi^\star) - V_M(\pi) - \frac{1}{\eta}\KL\left(M^{\nu(\cdot|\phi)}(\cdot|\pi), M^{\nu}(\cdot|\pi)\right)\right] \\ 
    &\leq 4H\sqrt{d \E_{\pi\sim p}\E_{\phi\sim \nu} 
 \left[\KL \left(M^{\nu(\cdot|\phi)}(\cdot|\pi), M^{\nu}(\cdot|\pi)\right)\right]} - \frac{1}{\eta} \E_{\pi\sim p}\E_{\phi\sim \nu}\left[\KL \left(M^{\nu(\cdot|\phi)}(\cdot|\pi), M^{\nu}(\cdot|\pi)\right)\right] \\
 &\leq 4\eta dH^2. 
 \end{align*}

    Finally, notice that the choice of $p$ in the definition of $\dec^{\textsc{KL}}_\eta(\calM,\Phi)$ allows it to achieve a smaller value than the choice of $p$ as a marginal distribution of $\nu$. This gives the desired bound. 
\end{proof}

%% file: appendix-adversarial.tex
\section{Omitted Proofs in \pref{sec: model-free adversarial}}
\label{app:model-free full proof}

\subsection{Specific instances that satisfies \pref{def:adv bilinear}}
In this section, we provide examples that satisfy \pref{def:adv bilinear}. Let $d_{P,h}^\pi$ be the occupancy measure induced by policy $\pi$, transition $P$ at step $h$.
Analog to Bellman residual decomposition, for any $\pi$ and $f \in \calF^{\pi}$, we have
\begin{align*}
    &f(\pi; R) - V_{P, R}(\pi) 
    \\&= \sum_{h=1}^H\left(\E_{(s_h,a_h) \sim d_{P, h}^\pi}\left[f(s_h, a_h; R)\right] - \E_{(s_{h+1},a_{h+1}) \sim d_{P, h}^\pi}\left[f(s_{h+1}, a_{h+1}; R)\right]\right) - \sum_{h=1}^H \E_{(s_h,a_h) \sim d_{P,h}^\pi}\left[R(s_h,a_h)\right]
    \\&= \sum_{h=1}^H \E^{\pi, P}\left[f(s_h, a_h; R) - R(s_h, a_h) - \E_{a_{h+1} \sim \pi(\cdot|s_{h+1})}\left[f(s_{h+1}, a_{h+1}; R)\right]\right]
\end{align*}
\paragraph{Low-rank MDPs} For low-rank MDPs, there exists feature $\varphi$ and $\psi$ such that the transition $P(s'|s,a) = \varphi(s,a)^\top \psi(s')$. We have 
\begin{align*}
    d_{P, h}^\pi(s) &= \sum_{s',a'} d_{P, h-1}^{\pi}(s',a') P(s|s', a')
    \\&= \sum_{s',a'} d_{P, h-1}^{\pi}(s',a') P(s|s', a')
    \\&= \sum_{s',a'} d_{P, h-1}^{\pi}(s',a') \left\langle \varphi(s',a'), \psi_h(s)\right\rangle
    \\&= \left\langle \underbrace{\sum_{s',a'} d_{P, h-1}^{\pi}(s',a')\varphi(s',a')}_{\varphi_h(\pi)}, \psi_h(s)\right\rangle
 \end{align*}
 Thus, 
 \begin{align*}
     &\E^{\pi, P}\left[f(s_h, a_h; R) - R(s_h, a_h) - \E_{a_{h+1} \sim \pi(\cdot|s_{h+1})}\left[f(s_{h+1}, a_{h+1}; R)\right]\right]
     \\&= \sum_{s_h,a_h} d_{P, h}^\pi(s_h)\pi(a_h|s_h)\left(f(s_h, a_h; R) - R(s_h, a_h) - \E_{s_{h+1} \sim P(\cdot|s_h, a_h)}\E_{a_{h+1} \sim \pi(\cdot|s_{h+1})}\left[f(s_{h+1}, a_{h+1}; R)\right]\right)
     \\&= \scalebox{0.95}{$\displaystyle\left\langle \underbrace{\varphi_h(\pi)}_{X_h(\pi, P)}, 
     \underbrace{\sum_{s_h, a_h} \psi_h(s_h)\pi(a_h|s_h)\left(f(s_h, a_h; R) - R(s_h, a_h) - \E_{s_{h+1} \sim P(\cdot|s_h, a_h)}\E_{a_{h+1} \sim \pi(\cdot|s_{h+1})}\left[f(s_{h+1}, a_{h+1}; R)\right]\right)}_{W_h(f, R; P)}\right\rangle $}
 \end{align*}
For any $\pi' \in \Pi$ and $f' \in \calF^{\pi'}$, any $o_h = (s_h, a_h, s_{h+1})$, define $$\ellest_h(f', o_h, R) = |\calA| \pi'(a_h|s_h)\left(f'(s_h, a_h; R) - R(s_h, a_h) - \E_{a_{h+1} \sim \pi'(\cdot|s_{h+1})}\left[f'(s_{h+1}, a_{h+1}; R)\right]\right).$$
Let $\pi_{\unif}$ be the uniform policy over $\calA$, and $\esttt(\pi) = \pi_{\unif}$ for every $\pi$, we have
\begin{align*}
    &\left\langle X_h(\pi; P), W_h(f', R, P)\right\rangle
    \\&= \sum_{s_h,a_h} d_{P, h}^\pi(s_h)\pi'(a_h|s_h)\left(f'(s_h, a_h; R) - R(s_h, a_h) - \E_{s_{h+1} \sim P(\cdot|s_h, a_h)}\E_{a_{h+1} \sim \pi'(\cdot|s_{h+1})}\left[f'(s_{h+1}, a_{h+1}; R)\right]\right)
    \\&= \E^{\pi\, \circ_h\, \pi_{\unif},\, P}\left[\ellest_h(f', o_h, R)\right]
     \\&= \E^{\pi\, \circ_h\, \esttt(\pi),\, P}\left[\ellest_h(f', o_h, R)\right].
\end{align*}
\paragraph{Low-occupancy MDPs} For low-occupancy MDPs, there exists features $\varphi$ and $\psi$ such that the state-action occupancy $d^\pi_{P,h}(s,a) = \varphi_h(\pi)^\top \psi_h(s, a)$. We have 
 \begin{align*}
     &\E^{\pi, P}\left[f(s_h, a_h; R) - R(s_h, a_h) - \E_{a_{h+1} \sim \pi(\cdot|s_{h+1})}\left[f(s_{h+1}, a_{h+1}; R)\right]\right]
     \\&= \sum_{s_h,a_h} d_{P, h}^\pi(s_h,a_h)\left(f(s_h, a_h; R) - R(s_h, a_h) - \E_{s_{h+1} \sim P(\cdot|s_h, a_h)}\E_{a_{h+1} \sim \pi(\cdot|s_{h+1})}\left[f(s_{h+1}, a_{h+1}; R)\right]\right)
     \\&= \scalebox{0.96}{$\displaystyle\left\langle \underbrace{\varphi_h(\pi)}_{X_h(\pi, P)}, 
     \underbrace{\sum_{s_h, a_h} \psi_h(s_h, a_h)\left(f(s_h, a_h; R) - R(s_h, a_h) - \E_{s_{h+1} \sim P(\cdot|s_h, a_h)}\E_{a_{h+1} \sim \pi(\cdot|s_{h+1})}\left[f(s_{h+1}, a_{h+1}; R)\right]\right)}_{W_h(f, R; P)}\right\rangle $}
 \end{align*}
For any $\pi' \in \Pi$ and $f' \in \calF^{\pi'}$, any $o_h = (s_h, a_h, s_{h+1})$, define $$\ellest_h(f', o_h, R) = \left(f'(s_h, a_h; R) - R(s_h, a_h) - \E_{a_{h+1} \sim \pi'(\cdot|s_{h+1})}\left[f'(s_{h+1}, a_{h+1}; R)\right]\right).$$
We have
\begin{align*}
    &\left\langle X_h(\pi; P), W_h(f', R, P)\right\rangle
    \\&= \sum_{s_h,a_h} d_{P, h}^\pi(s_h, a_h)\left(f'(s_h, a_h; R) - R(s_h, a_h) - \E_{s_{h+1} \sim P(\cdot|s_h, a_h)}\E_{a_{h+1} \sim \pi'(\cdot|s_{h+1})}\left[f'(s_{h+1}, a_{h+1}; R)\right]\right)
    \\&= \E^{\pi,\, P}\left[\ellest_h(f', o_h, R)\right].
\end{align*}
In this instance, $\esttt(\pi) = \pi$ for every $\pi$.

\subsection{Proof of \pref{thm:mf-f}}\label{app: odec}
In this section, we give a proof to \pref{thm:mf-f}, which is similar to \cite{foster2024model} with adaptation to more general partition. To start, we first define an extension to model-free optimistic DEC with bilinear divergence introduced in \cite{foster2024model}.  Given the definition of $f^\phi, \pi^\phi, \pi^\phi_{est}, \Phi$ and $D_{bi}^\pi$ discussed in \pref{sec: model-free adversarial}. We define
\begin{align*}
    \odec_{\eta}^{bi}(\calM, \Phi) = \max_{\rho \in \Delta(\Phi)}\min_{p \in \Delta(\Pi)}\max_{(P,R) \in \calM}\E_{\pi \sim p}\E_{\phi \sim \rho}\left[f^{\phi}(\pi^\phi, R) - V_{P,R}(\pi)  - \frac{1}{\eta} D_{bi}^{\pi}(\phi||P, R)\right].
\end{align*}
This complexity have the following bounds.  
\begin{lemma}
If $\pi^{\phi}_{est} = \pi^{\phi}$ for all $\phi$, we have
\begin{align*}
     \odec_{\eta}^{bi}(\calM, \Phi) \le \frac{\eta dH}{4}. 
\end{align*}
If $\pi^{\phi}_{est} \neq \pi^{\phi}$ for some $\phi$,  we have
\begin{align*}
     \odec_{\eta}^{bi}(\calM, \Phi) \le H\sqrt{\frac{\eta d}{2}}. 
\end{align*}
\label{lem:odec-bound}
\end{lemma}

\begin{proof}
For any $\rho \in \Delta(\Phi)$, we will bound
\begin{align*}
    &\min_{p \in \Delta(\Pi)}\max_{(P,R) \in \calM}\E_{\pi' \sim p}\E_{\phi \sim \rho}\left[f^{\phi}(\pi^{\phi}; R) - V_{(P, R)}(\pi')  - \frac{1}{\eta} D_{bi}^{\pi'}(\phi||P, R)\right] 
    \\&= \min_{p \in \Delta(\Pi)}\max_{\nu \in \Delta(\calM)}\E_{(P, R) \sim \nu}\E_{\pi' \sim p}\E_{\phi \sim \rho}\left[f^{\phi}(\pi^{\phi}; R) - V_{(P, R)}(\pi') - \frac{1}{\eta} D_{bi}^{\pi'}(\phi||P, R)\right]
    \\&= \max_{\nu \in \Delta(\calM)}\min_{p \in \Delta(\Pi)}\E_{(P, R) \sim \nu}\E_{\pi' \sim p}\E_{\phi \sim \rho}\left[f^{\phi}(\pi^{\phi}; R) - V_{(P, R)}(\pi') - \frac{1}{\eta} D_{bi}^{\pi'}(\phi||P, R)\right] \tag{\pref{lem:minmax}}
\end{align*}
For any $\nu \in \Delta(\calM)$, we choose $p$ such that $\pi' \sim p$ is equivalent to sampling $\phi' \sim \rho$ and play $\pi^{\phi'}_{\alpha}$ where $\pi^{\phi'}_{\alpha}$ is the randomized policy such that for every step $h$, play $\pi^{\phi'}$ with probability $1-\alpha/H$ and play $\pi_{est}^{\phi'}$ with probability $\alpha/H$. The policy choice for any step $h$ is independent. Thus, with probability $(1-\frac{\alpha}{H})^H \ge 1-\alpha$, $\pi^{\phi'}$ is played for every $H$. Using such $p$, since $V_M(\pi) \in [0,1]$ for any $M,\pi$, we only need to bound
\begin{align}
&\E_{\pi' \sim p}\E_{\phi \sim \rho}\left[f^{\phi}(\pi^{\phi}; R) - V_{(P,R)}(\pi')\right] \nonumber
\\&= \E_{\phi' \sim \rho}\E_{\phi \sim \rho}\left[f^{\phi}(\pi^{\phi}; R) - V_{(P,R)}(\pi^{\phi'}_{\alpha})\right] \nonumber
\\&\le \alpha  + \E_{\phi' \sim \rho}\E_{\phi \sim \rho}\left[f^{\phi}(\pi^{\phi}; R) - V_{(P,R)}(\pi^{\phi'}) \right]. 
\label{eq:model-free full}
\end{align}
For any $P, R$, we have
\begin{align*}
&\E_{\phi' \sim \rho}\E_{\phi \sim \rho}\left[f^{\phi}(\pi^{\phi}; R) - V_{(P,R)}(\pi^{\phi'}) \right]
\\&= \E_{\phi \sim \rho}\left[f^{\phi}(\pi^{\phi}; R) - V_{(P,R)}(\pi^{\phi}) \right]
    \\&\leq \sum_{h=1}^H  \E_{\phi \sim \rho}\left[\left|\langle X_h(\phi; P), W_h(\phi, R; P)\rangle\right|\right]. \tag{Bilinear property \pref{def:adv bilinear}}
\end{align*}
Define $\Sigma_h = \E_{\phi \sim \rho}\left[X_h(\phi; P)X_h(\phi; P)^\top\right]$ and let $\Sigma_h^{\dagger}$ be its pseudo-inverse, we have
\begin{align*}
    &\E_{\phi \sim \rho} \left[
\left|\langle X_h(\phi; P), W_h(\phi, R; P)\rangle\right|
\right] 
\\&= \E_{\phi \sim \rho}\left[
\left|\left\langle \left(\Sigma_h^\dagger\right)^{\frac{1}{2}}X_h(\phi; P), \Sigma_h^{\frac{1}{2}}W_h(\phi, R; P)\right\rangle\right|
\right]
\\&\le \sqrt{\E_{\phi \sim \rho}\left[\left\|\left(\Sigma_h^\dagger\right)^{\frac{1}{2}}X_h(\phi; P)\right\|_2^2\right]} \sqrt{\E_{\phi \sim \rho}\left[\left\|\Sigma_h^{\frac{1}{2}}W_h(\phi, R; P)\right\|_2^2\right]}
\end{align*}
We have
\begin{align*}
&\E_{\phi \sim \rho} \left[\left\|\Sigma_h^{\frac{1}{2}}W_h(\phi, R; P)\right\|_2^2\right] 
\\&= \E_{\phi \sim \rho}\left[\left\langle \Sigma W_h(\phi, R; P), W_h(\phi, R; P) \right\rangle\right]
\\&= \E_{\phi \sim \rho} \left[\left\langle \E_{\phi' \sim \rho}\left[X_h(\phi'; P)X_h(\phi'; P)^\top\right] W_h(\phi, R; P), W_h(\phi, R; P) \right\rangle\right]
\\&= \E_{\phi \sim \rho}\E_{\phi' \sim \rho} \left[\left\langle X_h(\phi'; P) , W_h(\phi, R; P)\right\rangle^2\right]
\end{align*}
Moreover, 
\begin{align*}
    \E_{\phi \sim \rho}\left[\left\|\left(\Sigma_h^\dagger\right)^{\frac{1}{2}}X_h(\phi; P)\right\|_2^2\right] \le d
\end{align*}
Thus, by Cauchy-Schwarz inequality, for any $P, R$, we have
\begin{align}
    &\E_{\phi' \sim \rho}\E_{\phi \sim \rho}\left[f^{\phi}(\pi^{\phi}; R) - V_{(P,R)}(\pi^{\phi'}) \right] \nonumber
    \\&\le \sqrt{ dH\sum_{h=1}^H \E_{\phi \sim \rho}\E_{\phi' \sim \rho} \left[\left\langle X_h(\phi';P) , W_h(\phi, R; P)\right\rangle^2\right]} \nonumber
    \\&= \sqrt{dH \sum_{h=1}^H \E_{\phi \sim \rho}\E_{\phi' \sim \rho}\left[\left( \E^{\pi^{\phi'}\,\circ_h \, \pi^{\phi'}_{est} , \,P}\left[\ellest_h(\phi, o_h; R)\right]\right)^2\right]} 
    \label{eq:bilinear est}
\end{align}
\textbf{(1)} If $\pi^\phi_{est} = \pi^\phi$ for all $\phi$, we can set $\alpha = 0$ and combine \pref{eq:bilinear est} and \pref{eq:model-free full} to conclude that for any $P,R$
\begin{align*}
    &\E_{\pi' \sim p}\E_{\phi \sim \rho}\left[f^{\phi}(\pi^{\phi}; R) - V_{(P,R)}(\pi') - \frac{1}{\eta} D_{bi}^{\pi'}(\phi||P, R)\right] \\
    &\leq \sqrt{dH \sum_{h=1}^H \E_{\phi \sim \rho}\E_{\phi' \sim \rho}\left[\left( \E^{\pi^{\phi'}\,\circ_h \, \pi^{\phi'}_{est} , \,P}\left[\ellest_h(\phi, o_h; R)\right]\right)^2\right]}  - \frac{1}{\eta} \E_{\phi' \sim \rho}\E_{\phi \sim \rho}\left[D_{bi}^{\pi^{\phi'}}(\phi||P, R)\right] \\
    &\leq \sqrt{dH \E_{\phi' \sim \rho}\E_{\phi \sim \rho}\left[D_{bi}^{\pi^{\phi'}}(\phi||P, R)\right]}  - \frac{1}{\eta} \E_{\phi' \sim \rho}\E_{\phi \sim \rho}\left[D_{bi}^{\pi^{\phi'}}(\phi||P, R)\right] \\
    &\leq \frac{\eta dH}{4}. 
\end{align*}
\textbf{(2)} If $\pi^\phi_{est} \neq \pi^\phi$ for some $\phi$, we need additional analysis. For any $h$,  $\pi_{\alpha}^{\phi'}$ select $\pi^{\phi'}$ with probability $1-\frac{\alpha}{H}$ and  select $\pi_{est}^{\phi'}$ with probability $\frac{\alpha}{H}$ independently. We define the event $\mathcal{E}_h$ as choosing $\pi^{\phi'}$ before step $h-1$ and choosing $\pi^{\phi'}_{est}$ at step $h$. We have $\mathbb{P}\left(\mathcal{E}_h\right) = \left(1-\frac{\alpha}{H}\right)^{h-1}\frac{\alpha}{H} \ge  \left(1-\frac{\alpha}{H}\right)^{H-1}\frac{\alpha}{H}$ from the independence. We have for any $h\in[H]$, 
\begin{align*}
    &\E_{\phi' \sim \rho}\E_{\phi \sim \rho}\left[\left(\E^{\pi_{\alpha}^{\phi'}, \,P}\left[\ellest_h(\phi, o_h; R)\right]\right)^2\right]
    \\& \ge \mathbb{P}\left(\mathcal{E}_h\right)\cdot \E_{\phi' \sim \rho}\E_{\phi \sim \rho}\left[\left(\E^{\pi^{\phi'}\,\circ_h \pi_{est}^{\phi'}, \,P}\left[\ellest_h(\phi, o_h; R)\right]\right)^2\right] 
    \\&\ge \left(1-\frac{\alpha}{H}\right)^{H-1}\frac{\alpha}{H}\E_{\phi' \sim \rho}\E_{\phi \sim \rho}\left[\left(\E^{\pi^{\phi'}\,\circ_h \pi_{est}^{\phi'}, \,P}\left[\ellest_h(\phi, o_h; R)\right]\right)^2\right]
    \\&\ge \frac{\alpha}{2H}\E_{\phi' \sim \rho}\E_{\phi \sim \rho}\left[\left( \E^{\pi^{\phi'}\,\circ_h \pi_{est}^{\phi'}, \,P}\left[\ellest_h(\phi, o_h; R)\right]\right)^2\right] \tag{$\alpha \le \frac{1}{2}$}. 
\end{align*}
Therefore, 
\begin{align*}
    \E_{\pi' \sim p}\E_{\phi \sim \rho}\left[D_{bi}^{\pi'}(\phi||P, R)\right] 
    &= \E_{\phi' \sim \rho}\E_{\phi \sim \rho}\left[\sum_{h=1}^H \left(\E^{\pi_{\alpha}^{\phi'}, \,P}\left[\ellest_h(\phi, o_h; R)\right]\right)^2\right] \\
    &\geq \frac{\alpha}{2H}\E_{\phi' \sim \rho}\E_{\phi \sim \rho}\left[\sum_{h=1}^H \left( \E^{\pi^{\phi'}\,\circ_h \pi_{est}^{\phi'}, \,P}\left[\ellest_h(\phi, o_h; R)\right]\right)^2\right]. 
\end{align*}

Combining this equation with \pref{eq:bilinear est} and \pref{eq:model-free full}, we have
\begin{align*}
    &\E_{\pi' \sim p}\E_{\phi \sim \rho}\left[f^{\phi}(\pi^{\phi}; R) - V_{(P,R)}(\pi')- \eta D_{bi}^{\pi'}(\phi||P, R)\right] \\
    &\le \alpha + \sqrt{dH \sum_{h=1}^H \E_{\phi \sim \rho}\E_{\phi' \sim \rho}\left[\left( \E^{\pi^{\phi'}\,\circ_h \, \pi^{\phi'}_{est} , \,P}\left[\ellest_h(\phi, o_h; R)\right]\right)^2\right]} \\
    &\qquad \qquad - \frac{1}{\eta} \cdot \frac{\alpha}{2H}\E_{\phi' \sim \rho}\E_{\phi \sim \rho}\left[\sum_{h=1}^H \left( \E^{\pi^{\phi'}\,\circ_h \pi_{est}^{\phi'}, \,P}\left[\ellest_h(\phi, o_h; R)\right]\right)^2\right].
    \\&\leq  \alpha + \frac{\eta dH^2}{2 \alpha}
\end{align*}
Thus, by setting $\alpha = H\sqrt{\eta d}$, we have
\begin{align*}
     \E_{\pi' \sim p}\E_{\phi \sim \rho_k}\left[f^{\phi}(\pi^{\phi}; R) - V_{(P,R)}(\pi') - \frac{1}{\eta} D_{bi}^{\pi'}(\phi||P, R)\right] \le H\sqrt{\frac{\eta d}{2}}. 
\end{align*}
\end{proof}
We restate \pref{thm:mf-f} here with more detailed dependencies with the language of partition defined in \pref{sec: model-free adversarial}. 

\noindent \textbf{Theorem 10} \ \ Assume $|\ellest_h(\cdot, \cdot, \cdot)| \le L$.    For bilinear class defined in \pref{def:adv bilinear}, \pref{alg:MAF} ensures with probability $1-\delta$: 
\begin{itemize}
    \item If $\pi^\phi_{\text{est}} = \pi^\phi$ for any $\phi$, then $ \Reg \le \order\left( HL\sqrt{d \log\left(T|\Phi|/\delta\right)} T^{\frac{3}{4}}\right)$.
    \item If $\pi^\phi_{\text{est}} \neq \pi^\phi$ for some $\phi$, then $\Reg \le \order\left(HL^{\frac{2}{3}}d^{\frac{1}{3}}\log\left(T|\Phi|/\delta\right)^{\frac{1}{3}}T^{\frac{5}{6}}\right)$.
\end{itemize}

\begin{proof}
We first analyze the regret of optimistic exponential weights, the proof is almost the same as \cite{foster2024model}. 
Choosing comparator as $\phi^\star = \phi_{f^\star, \pi^\star}$ where $f^\star = Q_{P^\star}^{\pi^\star}$. The reward function used in exponential weights is
\begin{align*}
    g_k(\phi) = \eta f^{\phi}(\pi^{\phi}; R_k) - \sum_{h=1}^H  \left(\frac{1}{\tau} \sum_{i \in [\tau]}\ellest_h(\phi, o_{k,h}^i, R_k)\right)^2.
\end{align*}
 Since $f(\pi, R) \leq [0,1]$ for any $f, \pi, R$, and $|\ellest_h(\cdot, \cdot, \cdot)| \le L$, we have $|g_k(\phi)| \le \eta  + HL^2$.  From \pref{lem:EXP}, if  $\gamma \le \frac{1}{4\eta  + 4HL^2}$ we have
\begin{align}
    &\sum_{k=1}^K \E_{\phi \sim \rho_k}\Bigg[\eta f^{\phi^\star}(\pi^{\phi^\star}; R_k) - \eta f^{\phi}(\pi^{\phi}; R_k) \nonumber
    \\& \quad + \sum_{h=1}^H \left(\frac{1}{\tau} \sum_{i \in [\tau]}\ellest_h(\phi, o_{k,h}^i, R_k)\right)^2-  \sum_{h=1}^H \left(\frac{1}{\tau} \sum_{i \in [\tau]}\ellest_h(\phi^\star, o_{k,h}^i, R_k)\right)^2\Bigg]  \nonumber
    \\&\le \frac{\log(|\Phi|)}{\gamma} + \gamma \sum_{k=1}^K \E_{\phi \sim \rho_k}\left[g_k(\phi)^2\right] \nonumber
    \\&\le \frac{\log(|\Phi|)}{\gamma} + \frac{1}{2} \sum_{k=1}^K \sum_{h=1}^H  \E_{\phi \sim \rho_k}\left[\left(\frac{1}{\tau} \sum_{i \in [\tau]}\ellest_h(\phi, o_{k,h}^i, R_k)\right)^2\right] + 2\eta^2 \gamma K.  \tag{$\gamma \le \frac{1}{4HL^2}$} \nonumber
\end{align}
Rearraging the inequality, we have
\begin{align}
&\sum_{k=1}^K \E_{\phi \sim \rho_k}\left[\eta f^{\phi^\star}(\pi^{\phi^\star}; R_k) - \eta f^{\phi}(\pi^{\phi}; R_k) + \frac{1}{2}\sum_{h=1}^H \left(\frac{1}{\tau} \sum_{i \in [\tau]}\ellest_h(\phi, o_{k,h}^i, R_k)\right)^2\right] \nonumber
\\&\le  \sum_{k=1}^K \sum_{h=1}^H \left(\frac{1}{\tau} \sum_{i \in [\tau]}\ellest_h(\phi^\star, o_{k,h}^i, R_k)\right)^2 +  \frac{\log(|\Phi|)}{\gamma} + 2\eta^2 \gamma K 
\label{eq:opt-exp-start}
\end{align}
By Hoeffding's inequality and union bound, with probability at least $1-\delta$, for any $k\in[K]$, $h\in[H]$, and any $\phi \in \Phi$, we have
\begin{align*}
    \left|\frac{1}{\tau} \sum_{i \in [\tau]}\ellest_h(\phi, o_{k,h}^i, R_k) -  \E^{ \pi_k,\, P^\star}\left[\ellest_h(\phi, o_h, R_k)\right]\right| \le L\sqrt{\frac{2\log\left(KH|\Phi|/\delta\right)}{\tau}} = \epsilon_{conc}. 
\end{align*}
Thus, given $D_{bi}^{\pi'}(\phi||P, R) = \sum_{h=1}^H  \left(\E^{\pi', P}\left[\ellest_h(\phi, o_h, R)\right]\right)^2$, we have 
\begin{align}
     \sum_{h=1}^H  \left(\frac{1}{\tau} \sum_{i \in [\tau]}\ellest_h(\phi, o_{k,h}^i, R_k)\right)^2 \le 2 D_{bi}^{\pi_k}(\phi||P^\star, R_k) + 2H\epsilon_{conc}^2.
     \label{eq:conc1}
\end{align}

\begin{align}
     \sum_{h=1}^H \left(\frac{1}{\tau} \sum_{i \in [\tau]}\ellest(\phi, o_{k,h}^i, R_k)\right)^2 \ge \frac{1}{2}D_{bi}^{\pi_k}(\phi||P^\star, R_k) -H\epsilon_{conc}^2. 
     \label{eq:conc2}
\end{align}
From Lemma A.3 in \cite{foster2021statistical}, with probability at least $1-\delta$, for any $\phi \in \Phi$, we have
\begin{align*}
   \frac{1}{2}\sum_{k=1}^K \E_{\pi \sim p_k}\left[D_{bi}^{\pi}(\phi||P^\star, R_k)\right] \le   \sum_{k=1}^K D_{bi}^{\pi_k}(\phi||P^\star, R_k) + 4HL^2\log(|\Phi|/\delta). 
\end{align*}
Combining \pref{eq:opt-exp-start}, \pref{eq:conc1} and \pref{eq:conc2}, and the fact that $D_{bi}^{\pi}(\phi^\star||P^\star, R) = 0$ for any $\pi$ and $R$, we have
\begin{align*}
     &\sum_{k=1}^K \E_{\phi \sim \rho_k}\left[\eta f^{\phi^\star}(\pi^{\phi^\star}; R_k) - \eta f^{\phi}(\pi^{\phi}; R_k) + \frac{1}{8} \E_{\pi \sim p_k}\left[D_{bi}^{\pi}(\phi||P^\star, R_k)\right]\right]  
    \\&\le \frac{\log(|\Phi|)}{\gamma} + \frac{2\gamma K}{\eta^2}  + \frac{6L^2KH\log\left(KH|\Phi|/\delta\right)}{\tau} +  4HL^2\log(|\Phi|/\delta)
\end{align*}
Choose $\gamma = \min\{\eta \sqrt{\frac{\log(|\Phi|)}{K}}, \frac{1}{4\eta + 4HL^2}\}$, we have 
\begin{align}
    &\sum_{k=1}^K \E_{\phi \sim \rho_k}\left[ f^{\phi^\star}(\pi^{\phi^\star}; R_k) - f^{\phi}(\pi^{\phi}; R_k) + \frac{1}{8\eta} \E_{\pi \sim p_k}\left[D_{bi}^{\pi}(\phi||P^\star, R_k)\right]\right]  \nonumber
    \\& \le \order\left(\sqrt{K\log(|\Phi|)} + HL^2\log(|\Phi|) +
 \frac{ L^2KH\log\left(KH|\Phi|/\delta\right)}{\eta\tau} +  \frac{1}{\eta} HL^2\log(|\Phi|/\delta)\right). 
 \label{eq:term1-bound}
\end{align}
Since  $f^{\phi^\star} = Q_{P^\star}^{\pi^\star}$ and $\pi^{\phi^\star} = \pi^\star$, we have $V_{(P^\star, R)}(\pi^\star) = f^{\phi^\star}(\pi^{\phi^\star}; R)$ for any $(P,R,\pi^\star) \in \phi^\star$. Regret can be decomposed to
\begin{align}
&\frac{1}{\tau}\Reg \nonumber
\\&=\sum_{k=1}^K V_{(P^\star, R_k)}(\pi^\star)  - \sum_{k=1}^K \E_{\pi \sim p_k}\left[V_{(P^\star, R_k)}(\pi)\right] \nonumber
\\&= \sum_{k=1}^K f^{\phi^\star}(\pi^{\phi^\star}, R_k)  - \sum_{k=1}^K \E_{\phi \sim \rho_k}\left[f^{\phi}(\pi^{\phi}; R_k)\right]
 + \sum_{k=1}^K \E_{\phi \sim \rho_k}\left[f^{\phi}(\pi^{\phi}; R_k)\right]  -  \sum_{k=1}^K \E_{\pi \sim p_k}\left[V_{(P^\star, R_k)}(\pi)\right] \nonumber
 \\&= \underbrace{\sum_{k=1}^K \E_{\phi \sim \rho_k}\left[ f^{\phi^\star}(\pi^{\phi^\star}; R_k) - f^{\phi}(\pi^{\phi}; R_k)  + \frac{1}{8\eta} \E_{\pi' \sim p_k}\left[D_{bi}^{\pi'}(\phi||P^\star, R_k)\right]\right]}_{\textbf{term1}} \nonumber
\\&\qquad + \underbrace{\sum_{k=1}^K \E_{\pi' \sim p_k}\E_{\phi \sim \rho_k}\left[f^{\phi}(\pi^{\phi}; R_k) - V_{(P^\star, R_k)}(\pi') - \frac{1}{8\eta} D_{bi}^{\pi'}(\phi||P^\star, R_k)\right]}_{\textbf{term2}} \label{eq:reg-dec}
\end{align}

\textbf{term1} is bounded through \pref{eq:term1-bound}. For \textbf{term2}, we have
\begin{align*}
\textbf{term2} &\le \sum_{k=1}^K \max_{(P,R) \in \calM}\E_{\pi' \sim p_k}\E_{\phi \sim \rho_k}\left[f^{\phi}(\pi^{\phi}; R) - V_{(P,R)}(\pi') - \frac{1}{8\eta} D_{bi}^{\pi'}(\phi||P, R)\right]
    \\&= \sum_{k=1}^K \underbrace{\min_{p \in \Delta(\Pi)}\max_{(P,R) \in \calM}\E_{\pi' \sim p}\E_{\phi \sim \rho_k}\left[f^{\phi}(\pi^{\phi}; R) - V_{(P,R)}(\pi')  - \frac{1}{8\eta} D_{bi}^{\pi'}(f, \pi||P, R)\right]}_{\le \odec^{bi}_{8\eta}(\calM, \Phi)}
\end{align*}
From \pref{lem:odec-bound}, we have
\textbf{(1)} If $\pi^{\phi}_{est} = \pi^{\phi}$, $\textbf{term2} \le O\big(\eta dH\big)$. Combining with \pref{eq:term1-bound} and regret decomposition \pref{eq:reg-dec}, we have 
\begin{align*}
    \frac{1}{\tau}\Reg \le \order\left(\eta dHK + \sqrt{K\log(|\Phi|)} + \frac{ L^2KH\log\left(KH|\Phi|/\delta\right)}{\eta\tau} +  \frac{1}{\eta} L^2 H\log(|\Phi|/\delta)\right)
\end{align*}
and thus (recall $T=K\tau$)
\begin{align*}
    \Reg \le \order\left(\eta dHT + \sqrt{T\tau\log(|\Phi|)} + \frac{ L^2TH\log\left(TH|\Phi|/\delta\right)}{\eta\tau} +  \frac{\tau}{\eta} L^2 H\log(|\Phi|/\delta)\right).
\end{align*}
Setting $\eta =  Ld^{-\frac{1}{2}} (\log(|\Phi|/\delta))^{\frac{1}{2}} T^{-\frac{1}{4}}$, $\tau = \sqrt{T}$, then $K = \frac{T}{\tau} = \sqrt{T}$, we have
\begin{align*}
    \Reg \le \tilde{\order}\left( HL\sqrt{d \log\left(|\Phi|/\delta\right)} T^{\frac{3}{4}}\right). 
\end{align*}
\textbf{(2)} If $\pi^{\phi'}_{est} \neq \pi^{\phi'}$, $\textbf{term2} \le \order\big(H\sqrt{\eta d}\big)$. 
Combining with \pref{eq:term1-bound} and regret decomposition \pref{eq:reg-dec}, we have 
\begin{align*}
    \frac{1}{\tau}\Reg \le \order\left(H\sqrt{\eta d}K + \sqrt{K\log(|\Phi|)} +  \frac{ L^2KH\log\left(KH|\Phi|/\delta\right)}{\eta \tau} +  \frac{1}{\eta} L^2H\log(|\Phi|/\delta)\right)
\end{align*}
and thus 
\begin{align*}
   \Reg \le \order\left(H\sqrt{\eta d}T + \sqrt{T\tau\log(|\Phi|)} +  \frac{ L^2TH\log\left(TH|\Phi|/\delta\right)}{\eta \tau} +  \frac{\tau}{\eta} L^2H\log(|\Phi|/\delta)\right).
\end{align*}
Setting $\eta = L^{\frac{4}{3}}d^{-\frac{1}{3}}(\log(|\Phi|/\delta))^{\frac{2}{3}} T^{-\frac{1}{3}}$, $\tau = \sqrt{T}$, then $K = \frac{T}{\tau} = \sqrt{T}$, we have
\begin{align*}
    \Reg \le \order\left(H L^{\frac{2}{3}}\left(d\log\left(|\Phi|/\delta\right)\right)^{\frac{1}{3}}T^{\frac{5}{6}}\right). 
\end{align*}

\end{proof}

%% file: appendix-support-lemma.tex
\section{Support Lemmas}

\begin{lemma}[Sion’s minimax theorem \citep{sion1958general}]
Let $\mathcal{X}$ and $\mathcal{Y}$ be convex and compact sets, and a function $F: \mathcal{X} \times \mathcal{Y} \rightarrow \mathbb{R}$ such that (1) For all $y \in \mathcal{Y}$, $F(\cdot, y)$ is convex and continuous over $\mathcal{X}$ and (2) For all $x \in \mathcal{X}$, $F(x, \cdot)$ is concave and continuous over $\mathcal{X}$. Then 
\begin{align*}
    \min_{x \in \mathcal{X}}\max_{y \in  \mathcal{Y}} F(x,y) =  \max_{y \in  \mathcal{Y}}\min_{x \in \mathcal{X}} F(x,y) 
\end{align*}
\label{lem:minmax}
\end{lemma}

\begin{lemma}[Performance Difference Lemma]
For any finite-horizon MDPs $M$ with state space $\mathcal{S}$, action space $\calA$, reward function $R$, transition $P$, and horizon length $H$, we have
\begin{align*}
    V_{(P, R)}(\pi') - V_{(P,R)}(\pi) = \sum_{h=1}^H \E_{s \sim d_{P,h}^\pi} \left[\sum_{a \in \calA} \left(\pi_h'(a|x) - \pi_h(a|x)\right) Q^{\pi'}_h(x,a; R)\right].
\end{align*}   
\label{lem:PDL}
\end{lemma}

\begin{lemma}[Exponential Weights]
Given a sequence of reward functions $\{g^t\}_{t=1}^T$ over a decision set $\Pi$,  $\{p^t\}_{t=1}^T$ is a distribution sequence with $p^t \in \Delta\left(\Pi\right), \,\forall t \in [T]$ such that 
\begin{align*}
    p^{t+1}(\pi) \propto \exp\left(\eta \sum_{i=1}^t g^i(\pi)\right).
\end{align*}
If $p^1$ is a uniform distribution over $|\Pi|$ and $\eta g^{t}(\pi) \leq 1$ for all $t \in [T]$ and $\pi \in \Pi$. Then
\begin{align*}
    \max_{p \in \Delta(\Pi)}\left\{\sum_{t=1}^T \left\langle g^t,  p - p^t\right\rangle \right\} \le \frac{\log(|\Pi|)}{\eta} + \eta\sum_{t=1}^T\sum_{\pi \in \Pi} p^t(\pi) g^t(\pi)^2.
\end{align*}
\label{lem:EXP}
\end{lemma}

Here we provide an extension of Lemma 5.2 in \cite{xu2023bayesian}. The proof follows \cite{xu2023bayesian} with adaptation to general $\Phi$.

\begin{lemma}[Lemma 5.2 in \cite{xu2023bayesian} for general $\Phi$]
Given $\rho \in \Delta(\Phi), \eta$, $p \in \Delta(\Pi)$, and let $\nu \in \Delta(\Psi)$ be a maximizer of
\begin{align*}
    &\nu \in \argmax_{\nu \in \Delta(\Psi)}\E_{\pi\sim p} \E_{(M,\pi^\star)\sim \nu} \E_{o\sim M(\cdot|\pi)}\left[ V_{M}(\pi^\star) -  V_{M}(\pi) - \frac{1}{\eta} \KL\left( \nu_{\bo{\phi}}(\cdot|\pi, o), \rho\right) \right]\,,
\end{align*}
then we have
\begin{align*}
    &\max_{\mu \in \Delta(\Psi)}\E_{\pi\sim p}\E_{\phi^\star \sim \mu}\E_{(M, \pi^\star) \sim \mu(\cdot|\phi^\star)} \E_{o\sim M(\cdot|\pi)}\left[ V_{M}(\pi^\star) -  V_{M}(\pi) - \frac{1}{\eta} \log\frac{\nu(\phi^\star|\pi, o)}{\rho(\phi^\star)} \right]\\
       &= \E_{\pi\sim p} \E_{(M,\pi^\star)\sim \nu} \E_{o\sim M(\cdot|\pi)}\left[ V_{M}(\pi^\star) -  V_{M}(\pi) - \frac{1}{\eta} \KL\left( \nu_{\bo{\phi}}(\cdot|\pi, o), \rho\right) \right]
\end{align*}
If the environment reveals $\pi^{\Theta} =(\pi^\theta)_{\theta\in\Theta}$, we also have
\begin{align*}
    &\max_{\mu \in \Delta(\calM)}\E_{\pi \sim p}\E_{\theta^\star \sim \mu}\E_{M\sim \mu(\cdot|\theta^\star)} \E_{o\sim M(\cdot|\pi)}\left[ V_{M}(\pi^{\theta^\star}) -  V_{M}(\pi) - \frac{1}{\eta} \log\frac{\nu(\theta^\star|\pi, o)}{\rho_t(\theta^\star)} \right]
    \\&= \E_{\pi\sim p} \E_{\theta^\star \sim \nu} \E_{M \sim \nu(\cdot|\theta^\star)} \E_{o\sim M(\cdot|\pi)}\left[ V_{M}(\pi^{\theta^\star}) -  V_{M}(\pi) -  \frac{1}{\eta} \KL\left( \nu_{\bo{\theta^\star}}(\cdot|\pi, o), \rho_t\right)\right]\,,
\end{align*}
as long as $\nu$ is a maximizer of the RHS.
\label{lem: max posterior}
\end{lemma}

\begin{proof}
Let $\mathcal{Z}$ be the space of all mappings from $\Pi \times \mathcal{O}$ to $\Delta(\Phi)$. For every mapping $Z \in \mathcal{Z}$, denote $Z(\cdot|\pi, o) \in \Delta(\Phi)
$ as the image of $(\pi, o)$. For a fixed $p$, define $B: \Delta(\Psi) \times \mathcal{Z} \rightarrow \mathbb{R}$ by
\begin{align*}
    B(\mu, Z) = \E_{\phi^\star \sim \mu}\E_{(M, \pi^\star) \sim \mu(\cdot|\phi^\star)}\E_{\pi \sim p}\E_{o \sim M(\cdot|\pi)}\left[V_M(\pi^\star) - V_M(\pi) - \frac{1}{\eta}\log\frac{Z(\phi^\star|\pi, o)}{q(\phi^\star)}\right]
\end{align*}
For any $\mu$, $B(\mu, Z)$ is strongly convex with respect to $Z$, For any $\mu \in \Delta(\Psi)$, denote $Z_\mu = \argmin_{Z \in \mathcal{Z}}  B(\mu, Z)$ as the unique minimizer of $B(\mu, z)$ given $\mu$. Solve the first order optimality condition for the Lagrangian of $B(\mu, Z)$ constrained on $Z$ as a valid distribution, we have $Z_{\mu}(\phi|\pi, o) = \mu(\phi|\pi, o)$ for every $\phi, \pi, o$ as the posterior of $\mu$. Thus, we have
\begin{align*} 
    &\min_{Z \in \mathcal{Z}} B(\mu, Z)
    \\&=  B(\mu, Z_{\mu})
    \\&= \E_{\phi^\star \sim \mu}\E_{(M, \pi^\star) \sim \mu(\cdot|\phi^\star)}\E_{\pi \sim p}\E_{o \sim M(\cdot|\pi)}\left[V_M(\pi^\star) - V_M(\pi) - \frac{1}{\eta}\log\frac{\mu(\phi^\star|\pi, o)}{q(\phi^\star)}\right]
\end{align*}
Define the worst-case prior belief 
\begin{align*}
    \nu \in \argmax_{\mu} B(\mu, Z_{\mu}) = \argmax_{\mu} \min_{Z \in \mathcal{Z}} B(\mu, Z).
\end{align*}
From a constructive minimax theorem (Lemma 5.1 in \cite{xu2023bayesian}), $(\nu, Z_{\nu})$ is a Nash equilibrium of function $B$, and $Z_{\nu}$ is a construction solution to the minimax optimization problem $\min_{Z}\max_{\mu}B(\mu, Z)$. Thus,
\begin{align*}
&\E_{\pi\sim p} \E_{(M,\pi^\star)\sim \nu} \E_{o\sim M(\cdot|\pi)}\left[ V_{M}(\pi^\star) -  V_{M}(\pi) - \frac{1}{\eta} \KL\left( \nu_{\bo{\phi}}(\cdot|\pi, o), \rho\right) \right]
    \\&= \E_{\phi^\star \sim \nu}\E_{(M, \pi^\star) \sim \nu(\cdot|\phi^\star)}\E_{\pi \sim p}\E_{o \sim M(\cdot|\pi)}\left[V_M(\pi^\star) - V_M(\pi) - \frac{1}{\eta}\log\frac{\nu(\phi^\star|\pi, o)}{q(\phi^\star)}\right]
    \\&= B(\nu, Z_{\nu})
    \\&= \max_{\mu} B(\mu, Z_{\nu})  \tag{Lemma 5.1 in \cite{xu2023bayesian}}
    \\&= \max_{\mu \in \Delta(\Psi)}\E_{\pi\sim p}\E_{\phi^\star \sim \mu}\E_{(M, \pi^\star) \sim \mu(\cdot|\phi^\star)} \E_{o\sim M(\cdot|\pi)}\left[ V_{M}(\pi^\star) -  V_{M}(\pi) - \frac{1}{\eta} \log\frac{\nu(\phi^\star|\pi, o)}{\rho(\phi^\star)} \right].
\end{align*}
The proof for the revealed $\pi^{\Theta}=(\pi^\theta)_{\theta\in\Theta}$ is similar.
\end{proof}